\documentclass[letterpaper]{article} %

\usepackage[preprint]{aaai2027}

\usepackage[hyphens]{url}
\usepackage{graphicx}
\usepackage{soul} %
\graphicspath{{./}}
\usepackage{natbib}
\usepackage{booktabs}
\usepackage{multirow}             %
\usepackage[table]{xcolor}        %
\usepackage{rotating}              %
\definecolor{tabband}{gray}{0.90} %
\usepackage{amsmath}
\usepackage{amssymb}
\usepackage{amsthm}                %
\usepackage{xspace}
\usepackage{tikz}                  %
\usetikzlibrary{positioning,calc,shapes.geometric,decorations.pathreplacing,arrows.meta}

\newtheorem*{theorem*}{Theorem}

\newtheorem*{lemma*}{Lemma}
\newtheorem{proposition}{Proposition}
\newtheorem*{proposition*}{Proposition}
\theoremstyle{definition}

\newcommand{\cgr}{\textsc{CGR}\xspace}
\newcommand{\Lconf}{\mathcal{L}_{\mathrm{conf}}}      %
\newcommand{\Lanchor}{\mathcal{L}_{\mathrm{anc}}}     %

\newcommand{\Sret}{$\mathcal{S}_{\mathrm{ret}}$}
\newcommand{\Stsk}{$\mathcal{S}_{\mathrm{tsk}}$}

\usepackage{tikz}

\title{Confine to the Task, Anchor the Rest: Concept-Guided Retention Regularizer for Parameter-Efficient Fine-Tuning}

\author{
    Ibna Kowsar\textsuperscript{\rm 1},
    Ivan Medri\textsuperscript{\rm 2},
    Fanyu Meng\textsuperscript{\rm 1},
    Xin Liu\textsuperscript{\rm 1}
}
\affiliations{
    \textsuperscript{\rm 1}Department of Computer Science, University of California, Davis\\
    \textsuperscript{\rm 2}University of Virginia\\
}

\begin{document}

\maketitle

\begin{abstract}
Parameter-efficient fine-tuning (PEFT) has been widely used, yet the adapted model suffers catastrophic forgetting of pre-trained knowledge. While most existing
work hardens the adapter by
constraining \emph{where} the weight update may lie, we note that constraining this weight channel does not constrain the activation channel.  In this paper, we propose Concept-Guided Retention (CGR), a plug-and-play regularizer that constrains both for any PEFT base. \cgr{} reads a frozen sparse autoencoder (SAE) and treats two disjoint concept sets differently. It softly confines the update to the task-concept subspace (via weight channel), and anchors a selected set of general-purpose concepts to the frozen model (via activation channel). We also provide anchor value bounds based on weight updates to understand how they interact, and show that both terms are necessary. 
On a CLIP vision panel, CGR significantly improves retention at comparable accuracy for every base, and the same pattern holds for world-knowledge retention under Llama-3.1-8B instruction tuning. Because the protected directions are \emph{named} SAE concepts, \cgr{} remains inspectable and controllable.\\[3pt]
\end{abstract}

\section{Introduction}

Parameter-efficient fine-tuning (PEFT) adapts frozen foundation models at a fraction of full fine-tuning's cost \citep{chen2022adaptformer}. Low-Rank Adaptation (LoRA)\citep{hu2022lora}, the de facto standard, parameterizes weight updates as $\Delta W=BA$. LoRA optimizes downstream loss without penalizing drift from $W_0$. The adapter is therefore free to write into directions that carry
general-purpose capability, and fine-tuned adapters measurably forget
pre-trained knowledge~\citep{biderman2024lora,shuttleworth2025lora}. The question is how an adapter can learn a new task while preserving pre-trained knowledge.

\textbf{The dominant answer is weight-geometric}, which is
one-sided. Recent work shows that initializing or confining the adapter within a
chosen \emph{weight} subspace mitigates catastrophic forgetting to a degree~\citep{oplora,loranull,corda,sclora,sails}. These methods differ in
\emph{how} they choose the subspace, but most of them constrain only \emph{where the
weight update lies}.
A subspace constraint only bounds the component of the weight delta outside
the chosen subspace and leaves the
drift of activation space unconstrained~\citep{cwac}. Thus, 
weight-subspace constraints alone are insufficient for
activation-space retention. Instead, we need to constrain both: a weight-subspace (where the update lies) and an
activation space (how retained features respond). 
To the best of our knowledge, no existing PEFT method constrains both, a key gap that motivates our work.

To address weight- and activation-space retention together, we ask
\emph{what} to constrain and \emph{how} to constrain it. We answer this with Sparse autoencoders (SAEs). SAEs have emerged as a tool for {editing} and {steering} concepts in a frozen model~\citep{marks2024sfc}. A key innovation is how we leverage a frozen SAE from a fixed layer to derive two disjoint concept sets: a task set and a concept-retention set. Using these two sets, our proposed  \textbf{\underbar{C}}oncept-\textbf{\underbar{G}}uided \textbf{\underbar{R}}etention (CGR), develops a plug-and-play regularizer with two terms: A \textsc{confine} term softly holds the
adapter weight-delta in the SAE task-concept subspace (the weight channel). An
\textsc{anchor} term holds the fine-tuned model's selected retention concepts (general-purpose features) to their pre-fine-tuning values (the
activation channel). 
 The construction is modality-agnostic, interpretable, and we validate it on vision and language tasks.

\paragraph{Contributions.}
\begin{itemize}
\item \textbf{Role-split retention.} 
From a frozen SAE we select two small disjoint concept sets: a task set, to which the weight update is confined, and a retention set, which is anchored to the frozen model. Anchoring a small set of the most frequently active off-task concepts beats blanketing every feature.

\item \textbf{A plug-and-play method.} We introduce two regularizers
that can be added to any PEFT adapter. Since adaptation is guided by named SAE concept basis, features remain interpretable after fine-tuning.

\item 
\textbf{Bounds on retention-concept drift.} %
We propose the \textsc{anchor} term as the expected drift of the retention concepts and derive bounds based on weight-space quantities (Prop.~\ref{prop:drift-bounds}). This derivation places weight projections and concept control in a common framework, clarifying how learned weight updates translate into movement of selected concepts and vice-versa.

\item \textbf{Evaluation.} We evaluate \cgr{} on ten vision tasks with CLIP ViT-B/32 and on instruction tuning with Llama-3.1-8B, each with six existing adapters. \cgr{} improves every adapter's
knowledge retention while maintaining comparable accuracy. Average vision retention rises by 2–7\% (LoRA: $74.1{\to}80.2$), and average LLM world-knowledge retention rises by 2–8\% (LoRA TriviaQA: $27.6{\to}43.5$).

\end{itemize}

\section{Related work}\label{sec:related}

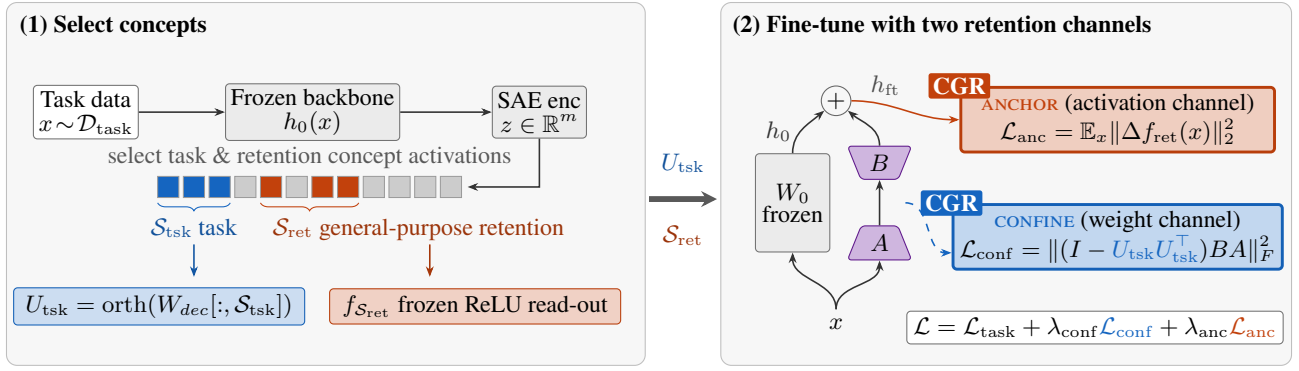
\begin{figure*}[t]
  \centering
  \providecommand{\Lconf}{\mathcal{L}_{\mathrm{conf}}}
\providecommand{\Lanchor}{\mathcal{L}_{\mathrm{anc}}}
\definecolor{cgrTask}{HTML}{1565C0}   %
\definecolor{cgrRet}{HTML}{C2410C}    %
\definecolor{cgrTrain}{HTML}{6A1B9A}  %
\begin{tikzpicture}[
  font=\small,
  >={Stealth[length=4pt]},
  box/.style    ={draw=black!55, fill=white, rounded corners=1.5pt,
                  inner sep=2.5pt, align=center},
  frozenbox/.style={box, fill=black!8},
  trainbox/.style ={box, fill=cgrTrain!32, draw=cgrTrain!90!black},
  chipT/.style  ={box, draw=cgrTask!95!black, fill=cgrTask!22, minimum width=3.85cm},
  chipR/.style  ={box, draw=cgrRet!95!black,  fill=cgrRet!24, minimum width=3.85cm},
  panel/.style  ={draw=black!25, fill=black!3, rounded corners=4pt},
  ptitle/.style ={anchor=north west, font=\small\bfseries, inner sep=2pt},
  note/.style   ={font=\small, text=black!65, align=center},
  arr/.style    ={->, black!80, line width=0.6pt},
  darr/.style   ={arr, dashed},
  flow/.style   ={-{Stealth[length=6pt, width=6pt]}, black!65, line width=2.2pt},
]
\draw[panel] (-2.60,0) rectangle (5.85,4.8);
\node[ptitle] at (-2.50,4.7) {(1)~Select concepts};
\node[box, inner sep=2pt, align=center] (dat) at (-1.55,3.35) {Task data\\[-1pt]$x\!\sim\!\mathcal{D}_{\mathrm{task}}$};
\node[frozenbox, inner sep=2pt] (enc) at (1.45,3.35) {Frozen backbone\\[-1pt]$h_0(x)$};
\node[frozenbox, inner sep=2pt] (sae) at (4.45,3.35) {SAE enc\\[-1pt]$z\in\mathbb{R}^m$};
\draw[arr] (dat) -- (enc);
\draw[arr] (enc) -- (sae);
\draw[arr] (sae.south) |- (3.50,2.36);
\node[note] at (1.41,2.75) {select task \& retention concept activations};
\foreach [count=\i] \c in {cgrTask,cgrTask,cgrTask,black!20,cgrRet,black!20,cgrRet,cgrRet,black!20,black!20,black!20,black!20}
  \draw[draw=black!40, line width=0.35pt, fill=\c]
    ({-0.60+(\i-1)*0.34},2.22) rectangle ++(0.28,0.28);
\draw[decorate, decoration={brace, mirror, amplitude=3pt}, cgrTask]
  (-0.60,2.15) -- (0.36,2.15);
\node[font=\small, text=cgrTask!80!black, anchor=north] at (-0.12,2.07)
  {$\mathcal{S}_{\mathrm{tsk}}$ task};
\draw[decorate, decoration={brace, mirror, amplitude=3pt}, cgrRet]
  (0.76,2.15) -- (2.06,2.15);
\node[font=\small, text=cgrRet!80!black, anchor=north west] at (0.80,2.07)
  {$\mathcal{S}_{\mathrm{ret}}$ general-purpose retention};
\node[chipT, inner sep=2.5pt] (us)  at (-0.58,0.78)
  {$U_{\mathrm{tsk}}=\operatorname{orth}(W_{dec}[:,\mathcal{S}_{\mathrm{tsk}}])$};
\node[chipR, inner sep=2.5pt] (ret) at (3.6,0.78)
  {$f_{\mathcal{S}_{\mathrm{ret}}}$ frozen ReLU read-out};
\draw[arr, cgrTask!85!black] (-0.12,1.6) -- (-0.12,1.17);
\draw[arr, cgrRet!85!black] (3.00,1.6) -- (3.00,1.17);
\draw[panel] (6.85,0) rectangle (14.40,4.8);
\node[ptitle] at (6.93,4.7) {(2)~Fine-tune with two retention channels};
\node[frozenbox, minimum width=0.80cm, minimum height=1.30cm] (w0) at (7.80,2.15)
  {$W_0$\\[-2pt]frozen};
\node[trainbox, trapezium, trapezium angle=68, minimum width=0.85cm, inner sep=2pt]
  (la) at (8.95,1.62) {$A$};
\node[trainbox, trapezium, trapezium angle=68, shape border rotate=180,
      minimum width=0.85cm, inner sep=2pt] (lb) at (8.95,2.68) {$B$};
\node (xin) at (8.375,0.55) {$x$};
\node[circle, draw=black!60, inner sep=0.6pt] (plus) at (8.375,3.50) {$+$};
\draw[arr] (xin.north) to[out=90,in=-90] (w0.south);
\draw[arr] (xin.north) to[out=90,in=-90] (la.south);
\draw[arr] (la.north) -- (lb.south);
\draw[arr] (lb.north) to[out=90,in=-40] (plus);
\draw[arr] (w0.north) to[out=90,in=-140] node[pos=0.45, left=0pt, font=\small] {$h_0$} (plus);
\node[chipT, align=center, line width=1.1pt, fill=cgrTask!26,
      minimum width=4.15cm] (conf) at (12.10,1.70)
  {{\color{cgrTask}\textsc{confine}} (weight channel)\\[1pt]
   $\Lconf=\lVert(I-{\color{cgrTask}U_{\mathrm{tsk}}U_{\mathrm{tsk}}^{\top}})BA\rVert_F^2$};
\node[fill=cgrTask, text=white, font=\small\bfseries, rounded corners=1pt,
      inner sep=1.6pt] at (conf.north west) {CGR};
\draw[darr, cgrTask] ($(la.east)!0.5!(lb.east)$) to[out=0,in=180] ([yshift=-6pt]conf.west);
\node[chipR, align=center, line width=1.1pt, fill=cgrRet!26,
      minimum width=4.15cm] (anch) at (12.10,3.25)
  {{\color{cgrRet}\textsc{anchor}} (activation channel)\\[1pt]
   $\Lanchor= \mathbb{E}_x\|\Delta f_{\mathrm{ret}}(x)\|_2^2$};
\node[fill=cgrRet, text=white, font=\small\bfseries, rounded corners=1pt,
      inner sep=1.6pt] at (anch.north west) {CGR};
\draw[arr, cgrRet] (plus.east)
  to[out=0,in=180] node[pos=0.30, above=0pt, font=\small, text=black!70] {$h_{\mathrm{ft}}$} (anch.west);
\node[box, inner sep=2.5pt] at (11.80,0.50)
  {$\mathcal{L}=\mathcal{L}_{\mathrm{task}}
    +\lambda_{\mathrm{conf}}{\color{cgrTask}\Lconf}
    +\lambda_{\mathrm{anc}}{\color{cgrRet}\Lanchor}$};
\draw[flow] (5.90,2.20) -- (6.80,2.20);
\node[font=\small, text=cgrTask!80!black] at (6.35,2.68) {$U_{\mathrm{tsk}}$};
\node[font=\small, text=cgrRet!80!black]  at (6.35,1.72) {$\mathcal{S}_{\mathrm{ret}}$};
\end{tikzpicture}
  \caption{{\cgr{} confines the weight update to task concepts
and anchors retained concepts to the frozen model, both read from one frozen
SAE.} {Left:} one pass of task data through the frozen backbone and SAE
selects the task set $\mathcal{S}_{\mathrm{tsk}}$ by causal attribution and a
disjoint retention set $\mathcal{S}_{\mathrm{ret}}$ by firing presence.
{Right:} two \cgr{} regularizers attach to any PEFT base.
{Confine} penalizes the part of $\Delta W$ outside the task subspace
and {anchor} holds the retention-concept read-out at its frozen-model
values.
}
  \label{fig:pipeline}
\end{figure*}

We group prior work by the space in which its constraint acts (weight space, activation space) and how SAEs are used. Each addresses catastrophic forgetting only
in part, and Table~\ref{tab:method_matrix} summarizes which mechanisms each
method covers. Concurrent
preprints are distinguished in
App.~\ref{app:concurrent}.

\begin{table}[t]\centering\small\tabcolsep=3pt

\begin{tabular}{l ccccccc}\toprule
Method     & PEFT & Weight & Activ. & Pers. & Split & Bound & V+L \\\midrule
OPLoRA     & \checkmark & {hard} & $\times$ & \checkmark & $\times$ & $\times$ & $\times$ \\
\rowcolor{tabband}
LoRA-Null  & \checkmark & {hard} & $\times$ & $\times$   & $\times$ & $\times$ & $\times$ \\
CorDA      & \checkmark & {hard} & $\times$ & $\times$   & $\times$ & $\times$ & $\times$ \\
\rowcolor{tabband}
SC-LoRA    & \checkmark & {hard} & $\times$ & $\times$   & $\times$ & $\times$ & $\times$ \\
SAILS      & \checkmark & {hard} & $\times$ & $\times$   & $\times$ & $\times$ & $\times$ \\
\rowcolor{tabband}
CWAC       & $\times$   & {hard} & {soft} & \checkmark & $\times$ & $\times$ & $\times$ \\
SAE-FD     & \checkmark & $\times$   & {soft} & \checkmark & $\times$ & $\times$ & $\times$ \\
\rowcolor{tabband}
\cgr{} (ours) & \checkmark & {soft} & {soft} & \checkmark & \checkmark & \checkmark & \checkmark \\
\bottomrule\end{tabular}

\vspace{2pt}
\caption{{Mechanism coverage} (as reported by each paper).
{Weight}\,/\,{Activ.}: form of the
weight-\,/\,activation-space retention term, hard (exact projection or
initialization) vs.\ soft (penalty); $\times$: channel unconstrained;
{Role split}: disjoint task and retention concept sets with opposite
treatments; {Bound}: a stated bound on the retained capability's
drop.}
\label{tab:method_matrix}
\end{table}

\textbf{Weight-space methods} confine the update $\Delta W$ to a chosen low-rank
subspace of each frozen weight. {OPLoRA}~\citep{oplora} projects out of the
top-$k$ singular subspace of $W_0$. {LoRA-Null}~\citep{loranull}
initializes in the null space of calibration activations.
{CorDA}~\citep{corda} uses an input-covariance--weighted SVD.
{SC-LoRA}~\citep{sclora} contrasts fine-tuning against a preserved
covariance, a formulation closely related to CorDA, and in this paper we use CorDA as a baseline for both. {SAILS}~\citep{sails} initializes the LoRA $B$ factor from
SAE decoder directions. Consistent with this picture, how much LoRA forgets is predicted by the
principal angles between the old- and new-task gradient subspaces~\citep{subgeom2026}.
As Table~\ref{tab:method_matrix} shows, every weight-channel constraint in
this cluster is a hard projection.  Some of these projections are continuously imposed whereas others are only done during initialization. This can be read as a jump between hard-or-nothing continual-control strategies.  
\cgr{}'s \textsc{confine} term is a soft continual version of these
projections filling the gap of previous hard-or-nothing approaches.
\cgr{}'s \textsc{anchor} term adds the
activation channel that these methods leave unconstrained
(\S\ref{sec:method}).

\textbf{Activation-space methods} in PEFT has received far less attention~\citep{vefa2025}.  Rather than constraining weights, these methods regularize internal representations via output distillation~\citep{lwf}, whole-feature alignment~\citep{kumar2022finetuning,ldifs}, or all-active SAE feature matching~\citep{saefd}. Beyond alignment, gradient projection techniques~\citep{dgsnet2025,lrcp2026} isolate updates for specialized tasks.  However, these existing anchors operate indiscriminately on raw hidden states or full feature maps, with no mechanism to select or name which concepts to preserve. In contrast, \cgr{} anchors a selected, interpretable subset of named concepts.

\textbf{Sparse autoencoders} decompose the residual stream into a dictionary of named concept directions~\citep{bricken2023monosemanticity}. Most work reads the dictionary to understand, steer, or edit a frozen
model~\citep{marks2024sfc,chalnev2024saets}, or adapts the model around an
SAE integrated into the forward path~\citep{chen2025loraforsae}. A few recent works bring the dictionary into the fine-tuning objective. SASFT~\citep{sasft} suppresses selected features during supervised
fine-tuning, and SSPU~\citep{sspu} unlearns target knowledge.
SAE-FD~\citep{saefd} distills \emph{all} active features under LoRA in
continual learning, and we evaluate it as a control
(\S\ref{sec:experiments}). CWAC~\citep{cwac} couples a hard weight projection with an SAE latent anchor to preserve safety concepts under \emph{full} fine-tuning. \cgr{} instead preserves pretrained knowledge under a PEFT adapter, splits one frozen dictionary into disjoint task and retention concept roles, and bounds the retention concept drift with its anchor term.

\section{Method}
\label{sec:method}

Concept-Guided Retention (\cgr) introduces two carefully designed regularizers for
adapter base, both read from a single frozen sparse autoencoder (Figure~\ref{fig:pipeline}). A \textsc{confine} term that keeps the low-rank weight update inside the
\emph{task}-concept subspace. An \textsc{anchor} term that holds the disjoint
\emph{retention}-concept activations to the frozen-model activations. The anchor equals the retained concepts' expected squared drift. We show that this drift is bounded between two weight-space quantities
(Prop.~\ref{prop:drift-bounds}). The lower side shows that a
weight-subspace constraint alone carries no retention guarantee, and the
upper side bounds the drift from the weights alone. \cgr{} minimizes the task's fine-tuning loss
$\mathcal{L}_{\mathrm{task}}$ (the classification loss for vision, the
language-modelling loss for text) together with the two auxiliary penalties,
\begin{equation}
  \mathcal{L}
  \;=\; \mathcal{L}_{\mathrm{task}}
  \;+\; \lambda_{\mathrm{conf}}\,\Lconf
  \;+\; \lambda_{\mathrm{anc}}\,\Lanchor
  \label{eq:total}
\end{equation}
with coefficients $\lambda_{\mathrm{conf}},\lambda_{\mathrm{anc}}\ge 0$.

\cgr{} makes deliberate
design choices. We constrain the weight space and the activation space
through two disjoint concept sets (\S\ref{sec:selection}). We confine
the weight update with a continual soft projection instead of hard projection (\S\ref{sec:confine}). We anchor the retention concepts'
activations because weight-space control alone does not constrain their drift(\S\ref{sec:anchor}). Both constraints act in a named concept basis, so the adapted features remain interpretable after fine-tuning.

\subsection{Preliminaries}
\label{sec:prelim}

\paragraph{Adapters for fine tuning.} Given $W_0\in \mathbb{R}^{d_{out}\times d_{in}}$ a pretrained matrix, an adapter adds a small
trainable update alongside it, $W_0+\Delta W$. For LoRA, this update is
low-rank with $r\ll\min(d_{out},d_{in})$.
\begin{equation}
  \Delta W \;=\; BA,
  \quad B\in\mathbb{R}^{d_{out}\times r},\;\; A\in\mathbb{R}^{r\times d_{in}}.
  \label{eq:lora}
\end{equation}
More generally, $\Delta W$ is any additive weight update. \cgr{} regularizes
the update $\Delta W$ and the activation change it causes on top of the frozen
model, and adds no trainable parameters of its own.
\paragraph{Sparse autoencoders.} A \emph{sparse autoencoder} (SAE) is a
trained map that re-describes a hidden activation $h\in\mathbb{R}^{d}$ as a
sparse combination of $m\gg d$ \emph{concept} directions. \cgr{} keeps this map
frozen and only reads from it. Two maps do this,
an \emph{encoder} that scores how strongly each concept is present, and a
\emph{decoder} that rebuilds $h$ from those scores,
\begin{equation}
  z(h) \;=\; {g\big(W_{enc}h+b_{enc}\big)}, \,
  \mathrm{dec}(z) \;=\; W_{dec} z + b_{dec}.
  \label{eq:sae}
\end{equation}
The score vector $z(h)\in\mathbb{R}^{m}$ represents \emph{concepts}, and it is sparse,
so most concepts are inactive for any given input. The gate $g$ is the
SAE-specific activation (i.e., \textsc{ReLU}, \textsc{TopK}, etc.). The decoder weight matrix
$W_{\!dec}\in\mathbb{R}^{d\times m}$ stores the $j$-th concept direction at column
$W_{\!dec}{[:,j]}$, and the encoder weight matrix $W_{\!enc}\in\mathbb{R}^{m\times d}$ at 
row $W_{enc}{[j,:]}$. The reconstruction term $W_{dec}\,z$ sums active concepts' directions.

\subsection{Reading concepts}
\label{sec:setup}
In a transformer, the \emph{residual stream} is a shared vector that every component
writes into and later components read from~\citep{elhage2021mathematical}. At a \emph{read-out point} in each layer, the frozen model produces a residual stream vector that %
we denote
$h_0(x)\in\mathbb{R}^{d}$ ($d=d_{out}$). Here, $x\in\mathbb{R}^{d_{in}}$ is
an internal representation from earlier layers (different from the raw
input) entering the last adapted weight matrix $W_0$.
Fine-tuning adds the adapter's contribution to this
vector, $h_{\mathrm{ft}}(x)=h_0(x)+\Delta W\,x$. Everything computed before is preserved in $h_0(x)$, and only $\Delta W x$ is new. This holds exactly
when $W_0$ writes directly into the stream and to first order otherwise.
\cgr{} works by limiting the effect of this added term at the read-out points. However, to define the concepts, we use an SAE that was pre-trained at a specific \emph{read-out point} on the residual stream of a fixed layer $\ell$.
\cgr{} reads \emph{how strongly} each concept fires. Each concept $j \in \mathcal{S}$ has an encoder row $W_{enc}{[j,:]}$ and bias
entry $b_{enc}{[j]}$. We set its \emph{concept read-out} to be a non-negative score, 
\begin{equation}
f_j(h)\;\triangleq\;\mathrm{ReLU}\big({W_{enc}{[j,:]}}h+b_{enc}{[j]}\big).
\label{eq:readout}
\end{equation}
We use a plain $\mathrm{ReLU}$ rather than the SAE's own gate $g$, so the read-out has the same form across SAE families. It also reads a concept whenever it is present, even where the SAE's gate would suppress it.

\subsection{Concept selection: task vs.\ retention sets}
\label{sec:selection}

Concept selection is key to \cgr{}. First, \cgr{} selects a
task set \Stsk{} that contains the concepts the fine-tuning task drives most strongly. It confines the weight update to the subspace these concepts span. Second, \cgr{} selects a retention set \Sret{} that contains the off-task general-purpose concepts on the weight update path. \cgr{} anchors the activations of these concepts to those of the frozen model {because their} activation drift drives forgetting {most}. Both sets are read from the task's own inputs $\mathcal{D}_{\mathrm{task}}$ on the frozen model, with no held-out data.

\paragraph{Task set.}An unconstrained update may write along off-task directions and cost the pretrained model its retention. The confined update must still
learn the task so the allowed subspace must span the directions that causally
drive the task loss. We select the concepts the task actually drives, measured by
their \emph{causal} effect on the task loss via attribution
patching~\citep{nanda2023attribution,marks2024sfc}. Specifically, each concept receives
an attribution score aggregated over the task inputs,
$a_j=\big|\mathbb{E}\big[\textstyle\sum_t z_j\,\partial\mathcal{L}_{\mathrm{task}}/\partial z_j\big]\big|$,
where $z_j$ is the SAE's own activation of concept $j$ and
$\mathcal{L}_{\mathrm{task}}$ is the fine-tuning loss of Eq.~\eqref{eq:total}.
The top-scoring concepts form the task set $\mathcal{S}_{\mathrm{tsk}}$. Stacking their decoder columns
and orthonormalizing them with a thin QR factorization~\citep{golub2013matrix} gives an
orthonormal basis of the task-concept subspace,
\begin{equation}
  U_{\mathrm{tsk}} \;=\; \mathrm{orth}\big(\,[\,W_{\!dec}{[:,j]}\,]_{j\in\mathcal{S}_{\mathrm{tsk}}}\,\big),
  \, U_{\mathrm{tsk}}^{\top}U_{\mathrm{tsk}} = I_{|\mathcal{S}_{\mathrm{tsk}}|}.
  \label{eq:qr}
\end{equation}
Here $U_{\mathrm{tsk}}\in\mathbb{R}^{d\times|\mathcal{S}_{\mathrm{tsk}}|}$
holds one orthonormal column per task concept. The \textsc{confine} term
uses this basis (\S\ref{sec:confine}).
\paragraph{Retention set.} For each concept $j$ we compute its \emph{mean activation}
on the frozen single-concept read-out $f_j$ of Eq.~\eqref{eq:readout},
averaging over the task inputs and over the $T$ token
positions of each input (patch tokens for vision, text tokens under the
attention mask for language),
\begin{equation}
  \phi_j \;=\; \mathbb{E}_{x\sim\mathcal{D}_{\mathrm{task}}}\Big[\tfrac{1}{T}\textstyle\sum_t f_{j}(h_0(x))_t\Big]
  \label{eq:firing}
\end{equation}
and form $\mathcal{S}_{\mathrm{ret}}$ from the $k$ concepts of largest $\phi_j$ that
lie outside the task set,
\begin{equation}
  \mathcal{S}_{\mathrm{ret}} \;=\; \mathrm{Top}\text{-}k\,\big\{\,\phi_j : j\notin\mathcal{S}_{\mathrm{tsk}}\,\big\}
  \label{eq:ret}
\end{equation}
where $\mathrm{Top}\text{-}k$ returns the indices of the $k$ largest scores.
A high $\phi_j$ marks a concept that fires on many of the task's inputs.
{
Fine-tuning writes $\Delta W$ along directions those same inputs activate,
so the highest-$\phi_j$ \textbf{off-task concepts (general-purpose features) are the
features most at risk that sit
directly in the update's path}. These are the
features the \textsc{anchor} holds
(\S\ref{sec:anchor}), through the stacked read-out
$f_{\mathcal{S}_{\mathrm{ret}}}(h)=\big(f_j(h)\big)_{j\in\mathcal{S}_{\mathrm{ret}}}$.
}

\subsection{The confine term}
\label{sec:confine}
We softly confine the update $\Delta W$ to the task-concept subspace by penalizing its perpendicular component. Writing $P_{\mathrm{tsk}}=U_{\mathrm{tsk}} U_{\mathrm{tsk}}^{\top}$ and $P_{\mathrm{tsk}}^{\perp}={I_d}-P_{\mathrm{tsk}}$,
\begin{equation}
  \Lconf
  \;=\; \big\|\,P_{\mathrm{tsk}}^{\perp}\,\Delta W\,\big\|_F^{2}
  \;=\; \big\|\,({I_d} - U_{\mathrm{tsk}} U_{\mathrm{tsk}}^{\top})\,BA\,\big\|_F^{2}
  \label{eq:confine}
\end{equation}
Since $\Delta W=BA$, its column space lies in $\mathrm{col}(B)$. The penalty therefore acts on the up-projection $B$. Minimizing $\Lconf$ steers the columns of $B$ towards the task directions. Since this is done for the updates $\Delta W$ that write directly in the residual stream, the confinement acts by backpropagation on the full transformer and feed-forward block updates. \cgr{}'s constraint is a \textbf{continual soft projection analogue of the hard subspace constraints in the weight-geometry family}, a penalty on the residual-facing factor applied at every step, rather than an exact projection (OPLoRA) or a subspace fixed once at initialization (LoRA-Null, SC-LoRA).

\begin{proposition} [Subspace confinement]
\label{prop:confine}
    Given a sequence of stationary points $\Delta W^q$ of
    $\mathcal{L}_{\mathrm{task}}(\Delta W)+\lambda_{\mathrm{conf}}^q\|P_{\mathrm{tsk}}^{\perp}\Delta W\|_F^2$, if $\|P_{\mathrm{tsk}}^\perp \nabla\mathcal{L}_{task}(\Delta W^q)\|\leq C$ for all $q$, then
    the perpendicular component obeys
    $\|P_{\mathrm{tsk}}^{\perp}\Delta W^q\|_F=O(1/\lambda_{\mathrm{conf}}^q)\to 0$ as
    $\lambda^q_{\mathrm{conf}}\to\infty$, recovering the hard projection
    $P_{\mathrm{tsk}}^\perp\Delta W = 0$.
\end{proposition}

The task subspace is built from the directions with the highest attribution
scores, so the projected gradient
$P_{\mathrm{tsk}}^\perp\nabla_{\Delta W}\mathcal{L}_{\mathrm{task}}$ retains
only the components off those selected directions (proof in
App.~\ref{app:pf-confine}). This makes the
bounded-gradient condition plausible in practice, although it does not by
itself guarantee a uniform bound along the sequence.
The attribution-based
choice of task directions and the confinement penalty then act in tandem to
control unnecessary drift along directions deemed less important to the
current task. Unlike a hard projection, which degrades performance when enforced strictly (Table~\ref{tab:ablation}), the soft penalty does not fully
eliminate updates outside the selected subspace, allowing the network to use
those directions when they provide sufficient improvement on the current task. Robustness with respect to the choice of
$\lambda_{\mathrm{conf}}$ in \cgr{} was empirically tested for several $\lambda_{\mathrm{conf}}$ values (Table~\ref{tab:lamconf}, App.~\ref{app:lambda}), leaving both accuracy and retention on a plateau above the respective limits of no \textsc{confine} term and hard-projection. In this way, the penalty helps preserve the pretrained capabilities the retention probes measure while retaining the plasticity required to learn the current task.

\paragraph{Why the {confine} term is not enough on its own.}
The \textsc{confine} penalty restricts
$\Delta W$, and that restriction alone lifts zero-shot retention well above
the unconstrained base (Table~\ref{tab:ablation}). However, retention is also related to the drift of the retention-concept \emph{activations},
$\Delta f_{\mathrm{ret}}(x)\triangleq
f_{\mathcal{S}_{\mathrm{ret}}}(h_{\mathrm{ft}}(x))
-f_{\mathcal{S}_{\mathrm{ret}}}(h_0(x))$, as shown in Figures~\ref{fig:drift_pred} and~\ref{fig:drift_ret}. Therefore, soft confine by itself is not sufficient for retention.

\paragraph{Why soft confine.} The lower bound in Proposition~\ref{prop:drift-bounds} shows that hard confine on the weight space, e.g., by projection, is also not enough to close the drift gap. Furthermore, our experiment shows that hard projection drives retention below the unconstrained base and adding
the \textsc{anchor} recovers only part of the loss
(Table~\ref{tab:ablation}).  The proven limitation of both soft and hard confine motivates the necessity of the anchor term, discussed next. 

\subsection{The {anchor} term}
\label{sec:anchor}
We anchor the read-out of the \emph{retention} concepts of the adapted model
to that of the frozen model,
\begin{equation}
  \Lanchor
  =\mathbb{E}_x\big\|
    f_{\mathcal{S}_{\mathrm{ret}}}(h_{\mathrm{ft}}(x))
    -f_{\mathcal{S}_{\mathrm{ret}}}(h_0(x))
  \big\|_2^2 .
  \label{eq:anchor}
\end{equation}
This is the expected squared retention drift
$\mathbb{E}_x\|\Delta f_{\mathrm{ret}}(x)\|_2^2$, and the next proposition
explains it from the weight side, showing along the way why no
weight-subspace constraint can control it.

The following proposition assumes that only the last layer
before the SAE read-out is adapted, so $x$ is a fixed internal
representation and $h_0(x)$ is the frozen output of that layer. The purpose of this
simplified setting is to obtain intuition about the relation between concept
directions and weight-update directions in a controlled manner. It does not
describe exactly what happens in practice, where earlier adapted layers also
modify the representation $x$, but it provides a model for understanding why
the proposed method may outperform weight-only constraints, as observed
empirically. It also shows why the weight subspace constraint of
Proposition~\ref{prop:confine} is not enough to control the
retention-concepts drift.

\begin{proposition}[Retention-Concepts and Weight-Change relation] \label{prop:drift-bounds}\label{prop:insufficiency} 
Let $W_{\mathrm{ret}}=W_{\mathrm{enc}}[\mathcal{S}_{\mathrm{ret}},:]$, and assume that the inputs have finite second moment $\Sigma=\mathbb{E}[xx^\top]$. Let $\mathcal{J}\subseteq\mathcal{S}_{\mathrm{ret}}$ and let $A$ be a subset of the input space such that, for every $x\in A$, all concepts in $\mathcal{J}$ are active at $h_0(x)$ and remain active under the perturbation $\Delta W x$. Write $W_{\mathcal{J}}=W_{\mathrm{ret}}[\mathcal{J},:]$ and define the restricted second moment $\Sigma_A=\mathbb{E}_x\!\left[\mathbf{1}_A(x)\,xx^\top\right]$.
Let us write
\begin{equation*}
  \mathcal{L}_{\mathrm{low}}
  =
  \mathrm{tr}\!\left(
    W_{\mathcal{J}}\Delta W\Sigma_A
    \Delta W^\top W_{\mathcal{J}}^\top
  \right),
\end{equation*}
\begin{equation*}
  \mathcal{L}_{\mathrm{up}}
  =
  \mathrm{tr}\!\left(
    W_{\mathrm{ret}}\Delta W\Sigma
    \Delta W^\top W_{\mathrm{ret}}^\top
  \right).
\end{equation*}
Then
\begin{equation}
  \mathcal{L}_{\mathrm{low}}
  \;\leq\;
  \mathbb{E}_x\big\|\Delta f_{\mathrm{ret}}(x)\big\|_2^2
  \;\leq\;
  \mathcal{L}_{\mathrm{up}}.
  \label{eq:drift-bounds}
\end{equation}
Moreover, even on the hard-projection surface $P_{\mathrm{tsk}}^\perp\Delta W=0$, neither $\mathcal{L}_{\mathrm{low}}$ nor $\mathcal{L}_{\mathrm{up}}$ is controlled by the projection  constraint {alone.} Whenever the task subspace contains directions read by the corresponding retention concepts and excited by the relevant input second moments {both bounds can grow}. A weight-subspace constraint alone is therefore insufficient to guarantee activation-space retention.
{(The proof of the proposition is in App.~\ref{app:pf-insuff})}
\end{proposition}

Eq \eqref{eq:drift-bounds} provides lower and upper bounds on the retention-concepts drift. The lower bound indicates that if a subset of the task data activates a retention concept, then movement restricted to the task-weight subspace may still amplify the drift of that concept. Although concepts correspond to separate coordinates in the latent space of the SAE, their encoder and decoder directions are not generally orthogonal when pulled back to the residual stream or weight spaces. This overlap is understandable because the retention concepts are themselves active on part of the task data. In this setting, the non-orthogonality may be advantageous: directly controlling retention-concept drift can steer movement toward task-relevant directions that learn the new task without unnecessarily disrupting the remaining concepts. 

CWAC~\citep{cwac} introduced the channel coupling via a
first-order argument for full fine-tuning in safety alignment tasks.
Prop.~\ref{prop:drift-bounds} strengthens it for the PEFT with exact bounds, and shows that confined updates can
\emph{amplify} drift. We anchor a disjoint off-task set
(\S\ref{sec:selection}) rather than one protected set. Further experimentally explored relations between weight space and concept activations are provided in App.~\ref{app:drift}.

\section{Experiments}
\label{sec:experiments}

For evaluation, we ask two questions. (i) \emph{Does adding \cgr{} improve retention over the
weight-only subspace-LoRA adapters?} (ii) \emph{Does the recipe transfer beyond
vision?} We evaluate on target-task accuracy and zero-shot retention on the
held-out probes after fine-tuning.
\begin{figure}[t]\centering
\includegraphics[width=\columnwidth]{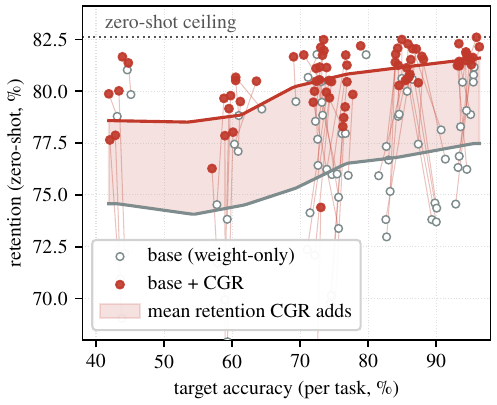}
\caption{{\cgr{} pushes the attainable-retention frontier toward the
ceiling (CLIP vision panel).} One point per (task, base) of
Table~\ref{tab:clip_vtab} where gray is for the base and red for ${+}$\cgr{}. The
shaded band shows the retention \cgr{} adds.}
\label{fig:vtab_frontier}
\end{figure}

\subsection{Setup}
\label{subsec:setup}

\paragraph{Backbones and SAEs.}
For vision, we adapt a frozen CLIP ViT-B/32 and read concepts from the frozen,
publicly released Prisma SAE~\citep{prisma2025} on the residual stream at
layer $\ell{=}8$. For language, we adapt Llama-3.1-8B and
read concepts from the frozen Llama-Scope SAE on the residual stream at layer
$19$ of $32$ \citep{he2024llamascope}. Both are intermediate read-out points, where semantic features are most
accessible~\citep{templeton2024scaling,lad2024stages}, before the final layers
specialize toward the output~\citep{gandelsman2024clip}. The
layer is fixed per modality, never tuned. App.~\ref{app:ev} reports
reconstruction faithfulness at the vision read-out layer for every task.
Both SAEs stay frozen and \cgr{} only reads from them, with
$|\mathcal{S}_{\mathrm{tsk}}|{=}|\mathcal{S}_{\mathrm{ret}}|{=}128$ for CLIP
and $|\mathcal{S}_{\mathrm{tsk}}|{=}256$,
$|\mathcal{S}_{\mathrm{ret}}|{=}128$ for Llama-3.1-8B. Neither size is tuned,
and App.~\ref{app:lambda} shows the effect of concept-set size
(Tables~\ref{tab:ksweep} and~\ref{tab:stasksweep}).

\begin{table*}[t]\centering\small\tabcolsep=1mm
\definecolor{cgrband}{gray}{0.90}
\begin{tabular}{l cc cc cc cc cc cc cc cc cc cc | c}\toprule
& \multicolumn{2}{c}{FGVC} & \multicolumn{2}{c}{CUB} & \multicolumn{2}{c}{DTD} & \multicolumn{2}{c}{Flow} & \multicolumn{2}{c}{Pets} & \multicolumn{2}{c}{EuroS} & \multicolumn{2}{c}{RESISC} & \multicolumn{2}{c}{DiabR} & \multicolumn{2}{c}{dSpr-O} & \multicolumn{2}{c}{SVHN} & \textbf{Avg} \\
\cmidrule(lr){2-3}\cmidrule(lr){4-5}\cmidrule(lr){6-7}\cmidrule(lr){8-9}\cmidrule(lr){10-11}\cmidrule(lr){12-13}\cmidrule(lr){14-15}\cmidrule(lr){16-17}\cmidrule(lr){18-19}\cmidrule(lr){20-21}
Method & Acc & Ret & Acc & Ret & Acc & Ret & Acc & Ret & Acc & Ret & Acc & Ret & Acc & Ret & Acc & Ret & Acc & Ret & Acc & Ret & Ret \\\midrule
Zero-shot
  & 19.3 & 82.6 & 70.9 & 82.6 & 54.6 & 82.6 & 73.0 & 82.6 & 89.8 & 82.6 & 40.5 & 82.6 & 61.9 & 82.6 & 35.4 & 82.6 & 9.3 & 82.6 & 60.4 & 82.6 & 82.6 \\
\midrule
LoRA
  & 59.3 & 67.9 & 74.6 & 70.1 & 71.5 & 74.1 & 93.2 & 76.3 & 82.6 & 73.8 & 93.6 & 78.3 & 83.0 & 76.7 & 72.3 & 78.6 & 43.6 & 70.9 & 89.4 & 73.8 & $74.1$ \\
\rowcolor{cgrband}
\;${+}$\cgr{}
  & 60.1 & 78.0 & 76.4 & 78.7 & 74.3 & 80.5 & 95.0 & 81.6 & 86.1 & 80.8 & 93.4 & 81.2 & 84.0 & 81.8 & 72.9 & 81.2 & 42.9 & 77.9 & 85.8 & 80.6 & $80.2$$^{\star}$ \\
OPLoRA
  & 58.9 & 69.9 & 75.1 & 69.3 & 72.1 & 72.1 & 93.4 & 76.9 & 81.5 & 75.9 & 95.1 & 78.9 & 82.7 & 77.3 & 72.6 & 76.4 & 44.2 & 72.2 & 89.8 & 74.6 & $74.4$ \\
\rowcolor{cgrband}
\;${+}$\cgr{}
  & 59.5 & 79.2 & 76.6 & 79.2 & 74.6 & 80.5 & 95.4 & 81.6 & 85.7 & 81.1 & 93.4 & 81.4 & 84.3 & 82.1 & 73.3 & 81.7 & 43.4 & 80.0 & 86.3 & 81.5 & $80.8$$^{\star}$ \\
LoRA-Null
  & 61.0 & 78.8 & 76.2 & 78.0 & 73.9 & 76.2 & 94.8 & 81.5 & 86.5 & 80.6 & 95.6 & 81.2 & 85.0 & 80.6 & 72.5 & 81.8 & 44.6 & 81.0 & 91.3 & 76.7 & $79.6$ \\
\rowcolor{cgrband}
\;${+}$\cgr{}
  & 60.5 & \textbf{80.5} & 76.9 & \textbf{80.5} & 76.7 & \textbf{81.7} & 95.9 & \textbf{82.6} & 87.1 & \textbf{82.0} & 93.3 & \textbf{82.3} & 85.0 & \textbf{82.5} & 73.5 & 82.0 & 43.9 & \textbf{81.7} & 87.9 & 81.7 & $\mathbf{81.7}$$^{\star}$ \\
CorDA
  & 57.8 & 71.7 & 72.5 & 65.9 & 73.3 & 78.8 & 93.9 & 80.4 & 84.4 & 78.8 & 94.5 & 79.1 & 84.6 & 78.9 & 73.2 & 79.5 & 43.1 & 78.8 & 90.1 & 74.4 & $76.6$ \\
\rowcolor{cgrband}
\;${+}$\cgr{}
  & 57.1 & 76.3 & 73.0 & 74.4 & 73.9 & 81.1 & 94.4 & 81.8 & 85.5 & 80.5 & 94.3 & 81.5 & 84.4 & 81.3 & 72.3 & 81.1 & 41.9 & 79.9 & 87.4 & 80.4 & $79.8$$^{\star}$ \\
AF
  & 60.9 & 77.1 & 77.1 & 75.9 & 74.7 & 76.0 & 95.5 & 80.8 & 87.4 & 77.4 & 95.5 & 80.4 & 85.7 & 80.0 & 72.8 & 80.5 & 45.1 & 79.8 & 90.7 & 78.1 & $78.6$ \\
\rowcolor{cgrband}
\;${+}$\cgr{}
  & 61.2 & 79.5 & 77.7 & 79.8 & 76.9 & 81.3 & 96.3 & 82.1 & 88.1 & 81.5 & 94.4 & 81.9 & 85.9 & 82.3 & 73.5 & \textbf{82.5} & 44.8 & 81.4 & 86.6 & \textbf{82.0} & $81.4$$^{\star}$ \\
SAILS
  & 59.0 & 65.8 & 74.4 & 69.7 & 71.0 & 72.4 & 92.9 & 74.6 & 82.7 & 73.0 & 94.5 & 76.2 & 83.2 & 75.2 & 72.6 & 77.7 & 43.8 & 69.0 & 90.0 & 73.7 & $72.7$ \\
\rowcolor{cgrband}
\;${+}$\cgr{}
  & 59.0 & 77.9 & 76.3 & 78.3 & 74.3 & 79.7 & 95.2 & 81.3 & 84.5 & 80.3 & 93.1 & 81.2 & 84.0 & 81.4 & 72.1 & 80.0 & 42.0 & 77.7 & 86.0 & 80.7 & $79.8$$^{\star}$ \\
\bottomrule\end{tabular}
\caption{{Results of \cgr{} on ten CLIP ViT-B/32 tasks.} Each adapter is shown with and without \cgr{}. AF is AdaptFormer. {Acc} is target-task accuracy. {Ret} is zero-shot retention, the mean accuracy on four held-out probes. {Avg} is the mean retention over the tasks. Best retention per task and best Avg in {bold}. The star ($^{\star}$) marks a statistically significant \cgr{} retention gain across tasks.}
\label{tab:clip_vtab}
\end{table*}

\paragraph{Baselines.}
We evaluate \cgr{} on six adapter families. Each base runs alone and as
``base${+}$\cgr{}'', so every comparison is a controlled ablation of the
\cgr{} terms with the adapter held fixed. \textbf{LoRA}~\citep{hu2022lora} and
the parallel-adapter \textbf{AdaptFormer}~\citep{chen2022adaptformer} are
retention-agnostic bases. \textbf{OPLoRA}~\citep{oplora},
\textbf{LoRA-Null}~\citep{loranull}, \textbf{CorDA}~\citep{corda}, and
\textbf{SAILS}~\citep{sails} are weight-geometry retention methods that
together cover the family's subspace constructions (top-$k$ SVD, activation null space, input covariance, and SAE directions). Anchors on the activation
channel are compared in the ablation of \S\ref{subsec:ablation}.
App.~\ref{app:baselines} describes each base.
\begin{table}[t]\centering\small\tabcolsep=3pt

\definecolor{cgrband}{gray}{0.90}
\begin{tabular}{l ccc cc c}\toprule
 & \multicolumn{5}{c}{World-knowledge retention (\%)} & Target \\
\cmidrule(lr){2-6}\cmidrule(lr){7-7}
Method & Trivia & NQ-O & WebQS & Avg1 & Per & MTB \\\midrule
Zero-shot & 61.5 & 7.9 & 11.6 & 27.0 & 100.0 & -- \\
\midrule
LoRA              & 27.6 & 11.1 & \textbf{12.4} & 17.1 & 81.6 & 5.8 \\
\rowcolor{cgrband}
\quad +\cgr{}     & 43.5 & 12.0 & \textbf{12.4} & 22.6 & 90.2 & 6.2 \\
LoRA-Null         & 47.0 & 9.5  & 9.8  & 22.1 & 87.1 & 3.5 \\
\rowcolor{cgrband}
\quad +\cgr{}     & 55.9 & 9.1  & 9.7  & \textbf{24.9} & 91.5 & 3.3 \\
CorDA             & 40.5 & 12.1 & 10.8 & 21.1 & 86.1 & 5.3 \\
\rowcolor{cgrband}
\quad +\cgr{}     & 45.7 & \textbf{13.2} & 9.4  & 22.8 & 85.0 & 5.3 \\
OPLoRA            & 40.2 & 10.2 & 11.5 & 20.6 & 88.0 & 5.7 \\
\rowcolor{cgrband}
\quad +\cgr{}     & 51.1 & 9.4  & 11.4 & 24.0 & \textbf{93.8} & 5.6 \\
SAILS             & 32.0 & 11.2 & 12.2 & 18.5 & 84.0 & 5.8 \\
\rowcolor{cgrband}
\quad +\cgr{}     & 42.9 & 11.6 & 12.1 & 22.2 & 89.9 & 5.7 \\
AdaptFormer       & 38.1 & 3.7  & 8.2  & 16.6 & 59.6 & 5.0 \\
\rowcolor{cgrband}
\quad +\cgr{}     & \textbf{59.4} & 6.2  & 9.9  & 25.1 & 86.7 & 4.4 \\
\bottomrule\end{tabular}
\caption{Results of \cgr{} on Llama-3.1-8B
instruction-following tasks. Each adapter is shown with and without \cgr{}. {Avg1} is
the average preservation score. {Per} is the percentage of preserved
performance relative to the pre-trained model. {MTB} is the MTBench
score. Trivia and NQ-O abbreviate TriviaQA and
NQ-Open.}\label{tab:llm_instruct}
\end{table}

\paragraph{Protocol.}
Vision methods fine-tune on ten classification tasks, eight from the
VTAB-1k benchmark~\citep{zhai2019vtab} and the two fine-grained tasks
FGVC-Aircraft~\citep{maji2013fgvc} and CUB-200~\citep{wah2011cub}. The language experiment follows
the instruction-tuning forgetting protocol of \citet{loranull}: fine-tune on WizardLM
instruction data and measure zero-shot world-knowledge retention on
TriviaQA, NQ-Open, and WebQS with \texttt{lm\_eval} 0.4.9.
Instruction-following quality is MTBench single-answer grading (1--10) with a Claude Opus 4.8 judge. The dataset-specific details are provided in
App.~\ref{app:datasets}.  All experiments are mean over three seeds, with
$\pm$std where shown. The LoRA-family adapters use rank $8$ for vision and rank $128$ for language,
where they share one trainable budget of $4.01\%$ of parameters. AdaptFormer
uses a bottleneck of $16$ and $64$ respectively. \cgr{} adds no parameters to
any base.
Every hyperparameter for
both modalities, shared across all bases per modality, is listed in
App.~\ref{app:hparams}.

\subsection{The accuracy--retention frontier}
\label{subsec:clip}
\paragraph{Vision.}
Retention is measured after fine-tuning as zero-shot accuracy on four
held-out \emph{natural} probes (CIFAR-10/100, Caltech-101, SUN397); the
frozen backbone scores $82.6$ on this set. Table~\ref{tab:clip_vtab} and
Figure~\ref{fig:vtab_frontier} present the per-task results. We observe that
\cgr{} improves retention on every base at comparable accuracy, and all
{$60$} dataset$\times$base cells improve (paired Wilcoxon
{$p\!\le\!0.002$} per base). Plain LoRA gains {$6.1$ points
($74.1\to80.2$)}. LoRA-Null, the strongest
weight-only baseline, still gains ({$79.6\to81.7$, within $0.9$ of the
ceiling}), and OPLoRA, itself a hard-projection retention method, gains
{$+6.4$} at matched accuracy. Across the 10 tasks, the one exception is accuracy on SVHN and dSprites,
two tasks that read out digit identity or synthetic pose rather than natural
object categories, and there \cgr{} gains retention at an accuracy cost. App.~\ref{app:probes} shows detailed vision panel retention performance for every base.

\paragraph{Language (Llama-3.1-8B).}
\label{subsec:llm}
We observe a similar trend on the language instruction-tuning task
(Table~\ref{tab:llm_instruct}), where retention improves on all six bases on average.
Vanilla LoRA
forgets the most (TriviaQA $61.6\to27.6$) and \cgr{} recovers nearly half of
the drop; \cgr{} on OPLoRA improves $40.2\to51.1$. MTBench performance stays close
to the original on five of the six bases. AdaptFormer is the one
exception, achieving the largest retention gain (TriviaQA $38.1\to59.4$) at an instruction-following cost ($5.03\to4.38$).
App.~\ref{app:llmstd} shows detailed retention performance for every base.

\subsection{Ablations}
\label{subsec:ablation}
\label{sec:ablation}

\begin{table}[t]\centering\small\tabcolsep=3pt

\newcommand{\msa}[2]{$#1$}  %
\begin{tabular}{l cc cc}\toprule
 & \multicolumn{2}{c}{FGVC} & \multicolumn{2}{c}{CUB} \\
\cmidrule(lr){2-3}\cmidrule(lr){4-5}
Configuration & Acc & Ret & Acc & Ret \\\midrule
LoRA (no \cgr{})              & \msa{59.3}{1.7} & \msa{53.1}{4.8} & \msa{74.6}{0.4} & \msa{59.5}{3.0} \\
\rowcolor{tabband}
Full \cgr{}$^*$                       & \msa{\mathbf{60.1}}{0.4} & \msa{\mathbf{72.0}}{0.8} & \msa{\mathbf{76.4}}{0.3} & \msa{\mathbf{73.2}}{0.7} \\
\midrule
Confine only & \msa{59.8}{0.9} & \msa{66.3}{1.8} & \msa{75.6}{0.2} & \msa{66.9}{1.4} \\
\rowcolor{tabband}
Anchor only & \msa{59.2}{0.7} & \msa{66.2}{4.4} & \msa{75.6}{0.5} & \msa{70.6}{1.1} \\
SAE-subspace init only & \msa{59.0}{1.5} & \msa{50.4}{5.4} & \msa{74.4}{0.3} & \msa{59.0}{3.1} \\
\midrule
Anchor set from CIFAR-100 & $59.7$ & $69.5$ & $76.0$ & $71.5$ \\
\rowcolor{tabband}
OPLoRA $+$ anchor & \msa{59.1}{1.2} & \msa{68.4}{2.5} & \msa{75.6}{0.3} & \msa{70.2}{0.5} \\
{Hard task-projection only} & {\msa{57.5}{0.6}} & {\msa{43.8}{4.4}} & {\msa{74.1}{0.6}} & {\msa{57.5}{0.9}} \\
{Hard task-projection $+$ anchor} & {\msa{58.4}{0.3}} & {\msa{63.8}{2.1}} & {\msa{75.2}{0.5}} & {\msa{69.4}{0.2}} \\
\bottomrule\end{tabular}
\caption{{Contribution of each \cgr{} component on top of a LoRA
base.} {Ret} is CIFAR-100 zero-shot retention.}\label{tab:ablation}
\end{table}

All vision ablations use FGVC-Aircraft and CUB-200 with CIFAR-100
zero-shot retention because these are the panel's two fine-grained tasks with their most severe
forgetting and in-distribution images for the SAE (App.~\ref{app:ev}). The
language model ablation keeps the protocol of \S\ref{subsec:setup}.

\paragraph{Both channels are crucial for retention.}
Table~\ref{tab:ablation} decomposes \cgr{} on a LoRA base. Initializing in an SAE subspace without a persistent loss changes nothing~\citep{sails}. Each
term of CGR alone lifts retention, and the two compose. The \textsc{anchor}
stops the retention concepts from drifting that \textsc{confine} cannot preserve (App.~\ref{app:drift}). The soft \textsc{confine} beats OPLoRA's hard projection under the same anchor. An anchor with hard task projection outperforms hard projection only.

\paragraph{Which concepts to anchor.}
Retention drops monotonically from top-$k$ to random to last-$k$ selection
(Table~\ref{tab:concept_id}). The selected concept set also outperforms anchoring
every active concept of~\citep{saefd,cwac} and distilling the whole
representation of~\citep{ldifs} on both downstream accuracy and retention. Crucially, retention performance is not dictated merely by the total magnitude of anchored activation:
our selected set anchors only ${\sim}53\%$ of the total concept activation yet achieves the highest retention.
Anchoring the entire representation retains less, while anchoring an equivalent activation mass on lower-scoring concepts performs no better than unselective anchoring. Retention-concept drift correlates with held-out retention across every run (Figure~\ref{fig:drift_pred}, App.~\ref{app:drift_ret}).
This ordering holds consistently on the language track (Table~\ref{tab:llm_concept_id}), confirming that the SAE provides tangible value through targeted selection and interpretability (\S\ref{subsec:interp}).

\begin{table}[t]\centering\small\tabcolsep=1mm

\newcommand{\ms}[2]{$#1$}  %
\begin{tabular}{l cc cc c}\toprule
 & \multicolumn{2}{c}{FGVC} & \multicolumn{2}{c}{CUB} & \\
\cmidrule(lr){2-3}\cmidrule(lr){4-5}
Anchor set & Acc & Ret & Acc & Ret &
  \begin{tabular}[c]{@{}c@{}}Anch.\\ act.\end{tabular} \\\midrule
Top-$k$ firing (ours)        & \ms{\mathbf{60.1}}{0.4} & \ms{\mathbf{72.0}}{0.8} & \ms{\mathbf{76.4}}{0.3} & \ms{\mathbf{73.2}}{0.7} & $53.3$ \\
\rowcolor{tabband}
Random selection             & \ms{59.6}{0.3} & \ms{69.2}{1.6} & \ms{\mathbf{76.4}}{0.2} & \ms{70.6}{0.9} & $3.6$ \\
Last-$k$ (least firing)      & \ms{59.4}{1.4} & \ms{64.8}{2.1} & \ms{76.1}{0.3} & \ms{67.9}{2.1} & ${<}0.1$ \\
\rowcolor{tabband}
SAE-FD~\citep{saefd}   & \ms{59.6}{1.2} & \ms{66.6}{1.6} & \ms{75.9}{0.0} & \ms{68.6}{1.7} & $100$ \\
All off-task concepts            & \ms{60.0}{1.3} & \ms{66.4}{1.6} & \ms{75.9}{0.1} & \ms{68.6}{1.6} & $98.3$ \\
\rowcolor{tabband}
All off-task except top-$k$         & \ms{58.3}{1.9} & \ms{65.4}{2.2} & \ms{76.0}{0.1} & \ms{68.8}{2.2} & $45.1$ \\
LDIFS~\citep{ldifs}                & \ms{59.5}{1.1} & \ms{64.8}{3.1} & \ms{75.9}{0.2} & \ms{67.2}{2.7} & -- \\
\bottomrule\end{tabular}
\caption{{Concept-selection ablation on CLIP for retention on top of LoRA Base.} {Ret} is CIFAR-100 zero-shot retention. Anchored activation (\%) is the share of the total SAE activation on the task inputs that the anchored set accounts for.}
\label{tab:concept_id}
\end{table}

\begin{table}[t]\centering\small\tabcolsep=3pt

\begin{tabular}{l ccc c}\toprule
Anchor set & Trivia & NQ-O & WebQS & Avg1 \\\midrule
Top-$k$ firing (ours) & \textbf{43.5} & \textbf{12.0} & \textbf{12.4} & \textbf{22.6} \\
\rowcolor{tabband}
Random selection & 43.4 & 10.7 & 12.1 & 22.1 \\
SAE-FD~\citep{saefd} & 42.1 & 11.1 & 12.1 & 21.7 \\
\bottomrule\end{tabular}
\caption{{Concept-selection ablation on Llama-3.1-8B for retention on top of LoRA base.}
}
\label{tab:llm_concept_id}
\end{table}

\begin{figure}[t]\centering
\includegraphics[width=\columnwidth]{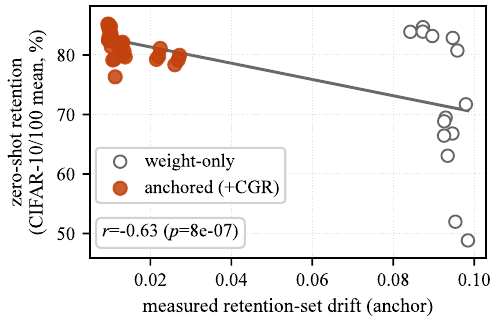}
\caption{The quantity the $\mathrm{anchor}$ reads predicts realized
retention (FGVC). Measured $\mathcal{S}_{\mathrm{ret}}$ drift vs.\ zero-shot
retention (CIFAR-10/100 mean) over all runs.}
\label{fig:drift_pred}
\end{figure}

\paragraph{Interpretability.}
\label{subsec:interp}
Because \cgr{}'s subspaces come from a frozen SAE rather than opaque singular
vectors or an activation null space, the protected retention directions are
\emph{named} and \emph{manipulable}. Figure~\ref{fig:interp} shows the
mechanism on FGVC and the weight-side complement is in
App.~\ref{app:interp_weight}. Removing the selected concepts costs
${\sim}10$pp of zero-shot retention while a size-matched random set costs
${\sim}0$pp (App.~\ref{app:ablation}). Steering them moves retention smoothly
with the intervention strength. Editing them through a concept-bottleneck
head~\citep{koh2020cbm} corrects predictions on demand. We discuss each in detail in App.~\ref{app:interp_weight}.

\begin{figure}[t]\centering
\includegraphics[width=\columnwidth]{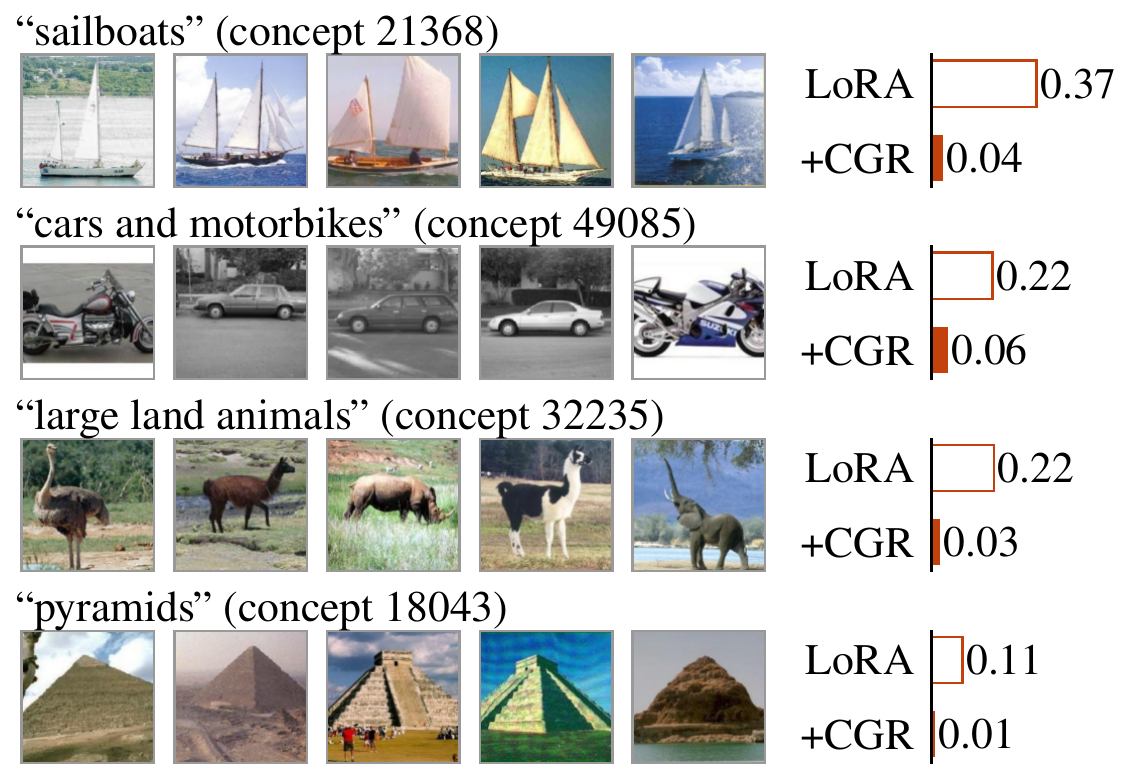}
\caption{{\cgr{} pins the concepts that drift most under plain LoRA (FGVC).} Each concept is named by its top-firing Caltech-101 exemplars. Bars show the activation drift under LoRA (gray) and under ${+}$\cgr{}
(red).}
\label{fig:interp}
\end{figure}

\section{Conclusion}
\label{sec:conclusion}

We propose \cgr{}, a plug-and-play regularization method that
uses the concepts of a frozen sparse autoencoder to preserve pre-trained
knowledge during parameter-efficient fine-tuning. A $\textsc{confine}$
term restricts the weight update to the task-concept subspace, and an
$\textsc{anchor}$ term holds a selected, disjoint set of general-purpose concept activations to the frozen model. Extensive experiments across six adapter families on a CLIP vision panel and a Llama-3.1 instruction protocol show that \cgr{} improves retention on every base at comparable task accuracy. Furthermore, each term lifts retention in different channels, and the $\textsc{anchor}$ contributes to the part that the weight-space update does not hold. Additionally, we show that \cgr{}'s preserved capability can be inspected, steered, and edited after fine-tuning. Future work includes a per-task certificate for sequential continual learning, bounds under distribution shift, and drawing both concept sets from SAEs at several layers.

\bibliography{references,aaai2026,cgr}

\clearpage
\appendix

\section{Concurrent work}\label{app:concurrent}

Two concurrent preprints overlap parts of \cgr{}'s design.
\textbf{SAE-FT}~\citep{saeft2026} penalizes the addition and removal of SAE
features for robust fine-tuning, on a single feature set in the activation
channel alone. \textbf{PRISM}~\citep{prism2026} bounds a cross-entropy risk gap
by whole-representation drift, without concept selection or a weight-space
certificate. Neither pairs the role split with a weight-space term, and neither
attaches to the retention-focused bases we evaluate.

\section{Experimental Details}\label{app:expdetails}

\subsection{Baselines}\label{app:baselines}

\textbf{LoRA}~\citep{hu2022lora} learns a low-rank update
$\Delta W{=}BA$ on frozen pretrained weights. It places no constraint on
where the update acts, so it is a retention-agnostic base.

\textbf{AdaptFormer}~\citep{chen2022adaptformer} inserts a parallel
bottleneck adapter (down-projection, nonlinearity, up-projection) beside the
MLP block. It is the one non-LoRA adapter in the panel and is likewise
retention-agnostic.

\textbf{OPLoRA}~\citep{oplora} applies a double-sided projection that
keeps the update out of the top-$k$ singular subspace of the pretrained
weight, which leaves the dominant spectral directions of the base model
untouched.

\textbf{LoRA-Null}~\citep{loranull} initializes the adapter in the
null space of calibration activations, so the update starts with no effect
on the calibration inputs.

\textbf{CorDA}~\citep{corda} builds the adapter from an
input-covariance weighted SVD of the pretrained weight. The covariance comes from general data in its knowledge-preserving mode and from task data in its instruction-preserving mode. 

\textbf{SAILS}~\citep{sails} initializes the LoRA $B$ factor from SAE
decoder directions, which starts the update in an interpretable concept
basis.

We do not include Wise-FT~\citep{wortsman2022wiseft}. It leaves fine-tuning
unconstrained and interpolates the fine-tuned weights back toward the
pre-trained ones afterward, so it composes with any of the methods we
compare and leaves the ranking between them unchanged~\citep{ldifs}.

\subsection{Datasets}\label{app:datasets}

Table~\ref{tab:datasets} lists the vision datasets with their class
counts. The panel spans fine-grained recognition,
textures, natural categories, remote sensing, medical imaging, and
structured pose and digit recognition. Eight tasks (DTD, Flowers-102,
Oxford-Pets, EuroSAT, RESISC45, Diabetic Retinopathy, dSprites-Orientation,
SVHN) come from the VTAB benchmark~\citep{zhai2019vtab} and train on $1000$
examples each under its VTAB-1k protocol. The two fine-grained tasks,
FGVC-Aircraft~\citep{maji2013fgvc} and CUB-200~\citep{wah2011cub}, train on
their full training sets. Model
selection is based on the best validation accuracy on a held-out $20\%$ split of the
training budget. Retention is measured after fine-tuning as zero-shot
accuracy on four held-out natural probes, CIFAR-10,
CIFAR-100~\citep{krizhevsky2009cifar},
Caltech-101~\citep{feifei2004caltech}, and SUN397~\citep{xiao2010sun}. No
probe is ever a fine-tuning task.

\begin{table*}[t]\centering\small\tabcolsep=4.5pt
\begin{tabular}{l cccccccccc cccc}\toprule
& \multicolumn{10}{c}{\textbf{Fine-tuning tasks}} & \multicolumn{4}{c}{\textbf{Retention probes}} \\
\cmidrule(lr){2-11}\cmidrule(lr){12-15}
& FGVC & CUB & DTD & Flow & Pets & EuroS & RESISC & DiabR & dSpr-O & SVHN & C-10 & C-100 & Calt-101 & SUN397 \\\midrule
Classes & $100$ & $200$ & $47$ & $102$ & $37$ & $10$ & $45$ & $5$ & $16$ & $10$ & $10$ & $100$ & $101$ & $397$ \\
\bottomrule\end{tabular}
\caption{{Vision task datasets.} Class counts for the ten fine-tuning
tasks and the four held-out retention probes, with column abbreviations as
in Table~\ref{tab:clip_vtab}. Training budgets are given in the text.
Probes are evaluated zero-shot only and are never fine-tuning tasks.}
\label{tab:datasets}
\end{table*}

\section{Training Details and Hyperparameters}\label{app:hparams}
\paragraph{Compute.}
Vision and language experiments run on separate machines.
Vision fine-tuning jobs use an NVIDIA RTX 2080\,Ti (11\,GB) on a shared
8-GPU node. Language jobs use an NVIDIA GH200 (96\,GB) with 16 CPU cores.
The vision node runs Linux with Python~3.12, PyTorch~2.12.1 (CUDA~13.0),
and \texttt{open\_clip\_torch}~3.3.0. The language node runs Linux (aarch64)
with Python~3.11, PyTorch~2.7.1 (CUDA~12.8), \texttt{transformers}~5.8.1,
\texttt{peft}~0.19.1, \texttt{sae-lens}~6.43.0, and \texttt{lm-eval}~0.4.9.

\paragraph{Training configuration.}
Table~\ref{tab:hparams} lists every hyperparameter for both modalities. All
adapter families share one training configuration per modality, so each
base-vs-\cgr{} comparison in the main tables varies only the \cgr{} terms. Unless marked otherwise, every vision number in this appendix is the mean over
$3$ seeds ($\pm$ sample std where shown). Language retention numbers are
$3$-seed means. Each configuration runs with
seeds $0$, $1$, and $2$.

The language track fine-tunes Llama-3.1-8B on the WizardLM
evolved-instruction corpus~\citep{xu2024wizardlm} with sequences capped at
$512$ tokens for a fixed budget of $780$ steps at an effective batch size of
$128$. Retention is zero-shot accuracy on the full test sets of
TriviaQA~\citep{joshi2017triviaqa}, NQ-Open~\citep{kwiatkowski2019nq}, and
WebQS~\citep{berant2013webquestions}. Instruction-following quality is
MTBench single-answer grading~\citep{zheng2023judging} with a fixed judge. The language model optimization configuration follows \citet{loranull}: AdamW, learning rate
$2{\times}10^{-5}$, cosine schedule with $0.03$ warmup, effective batch
$128$.

\paragraph{Coefficient selection.}
\cgr{} itself adds two loss coefficients and the two concept-set sizes. We did not tune coefficients per task. The one anchor coefficient
$\lambda_{\mathrm{anc}}$ and confine coefficient $\lambda_{\mathrm{conf}}$
were fixed from a single elbow sweep (App.~\ref{app:lambda}) and reused
everywhere. Retention is insensitive to the retention-set size over
$|\mathcal{S}_{\mathrm{ret}}|\in[32,256]$, and accuracy is likewise insensitive
to the task-set size over $|\mathcal{S}_{\mathrm{tsk}}|\in[32,512]$
(App.~\ref{app:lambda}).

\begin{table*}[t]\centering\small\tabcolsep=5pt

\begin{tabular}{l ll}\toprule
& \textbf{Vision (CLIP)} & \textbf{Language (Llama)} \\\midrule
\multicolumn{3}{l}{\emph{Backbone \& frozen SAE}} \\
Backbone            & CLIP ViT-B/32        & Llama-3.1-8B (bf16) \\
SAE site            & resid., block $8$    & resid., layer $19/32$ \\
SAE                 & Prisma ($64\times$, ReLU$+$L1) & Llama-Scope ($8\times$, top-$32$) \\
$|\mathcal{S}_{\mathrm{tsk}}|,|\mathcal{S}_{\mathrm{ret}}|$ & $128,128$ & $256,128$ \\
\midrule
\multicolumn{3}{l}{\emph{Optimization}} \\
Optimizer           & AdamW                & AdamW \\
Learning rate       & $1{\times}10^{-3}$   & $2{\times}10^{-5}$ \\
Weight decay        & $10^{-4}$            & $0$ \\
Schedule            & cosine, $5$-ep warmup & cosine, $0.03$ warmup ratio \\
Budget              & $50$ epochs          & $780$ steps \\
Batch (effective)   & $64$                 & $16{\times}8{=}128$ \\
Grad clip           & --                   & $1.0$ \\
Selection           & best val.\ acc.\ ($20\%$ split) & fixed step budget \\
Data regime         & $1000$-shot (VTAB); full (FG) & WizardLM SFT, len $512$ \\
Seeds               & $3$ ($0,1,2$)        & $3$ ($0,1,2$; MTBench: $1$) \\
\midrule
\multicolumn{3}{l}{\emph{Adapters (rank / config)}} \\
LoRA                & $r{=}8$              & $r{=}128,\alpha{=}256$ \\
OPLoRA              & $r{=}8, k{=}128$     & $r{=}128$ \\
LoRA-Null           & $r{=}8$              & $r{=}128$ (NQ-Open calib, $256$) \\
CorDA-IPM           & $r{=}8$              & $r{=}128$ \\
AdaptFormer         & bottleneck $16$      & bottleneck $64$, scale $0.1$ \\
SAILS-init          & $r{=}8, \alpha{=}0.1$ & $\alpha{=}0.1$ \\
\midrule
\multicolumn{3}{l}{\emph{\cgr{} (added terms)}} \\
Anchor $\lambda_{\mathrm{anc}}$   & $60$   & $0.2$ \\
Confine $\lambda_{\mathrm{conf}}$ & $0.1$  & $0.2$ \\
Confine mode        & inside-task          & inside-task \\
Retention set       & top firing           & top firing \\
\bottomrule\end{tabular}
\caption{{Hyperparameters.} Shared across all adapter families unless noted;
\cgr{} adds only the two loss coefficients and the concept-set sizes. The anchor coefficient
$\lambda_{\mathrm{anc}}$ differs across modalities because the anchor is a mean
squared {activation} drift, whose scale is set by the backbone's residual-stream
magnitude, not a tuned trade-off.}
\label{tab:hparams}
\end{table*}

\begin{figure*}[t]\centering
\includegraphics[width=\textwidth]{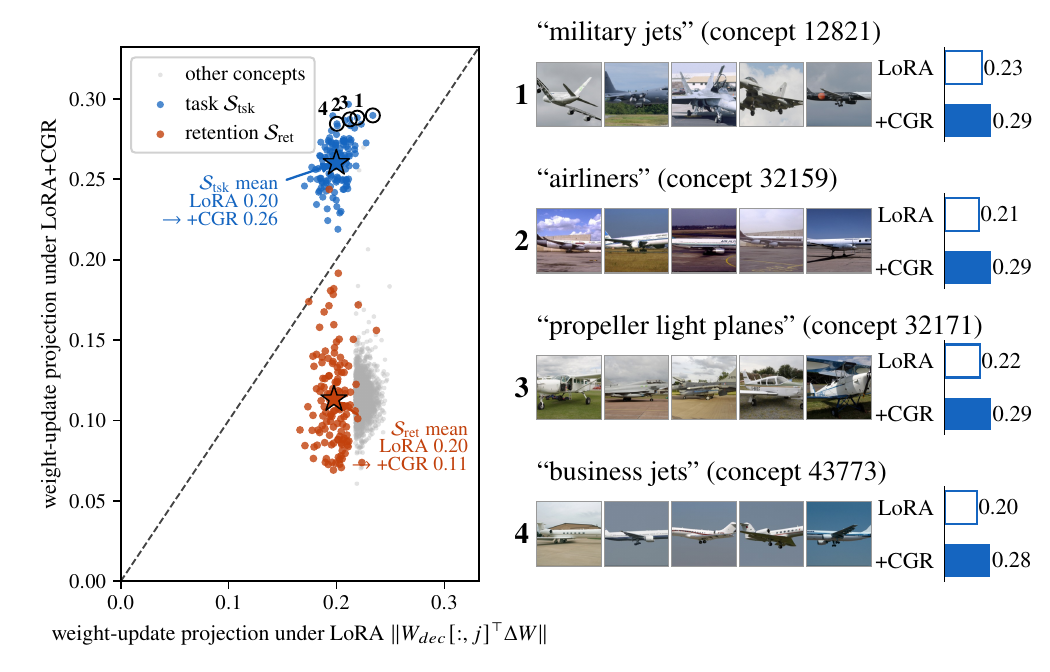}
\caption{{The weight update concentrates on the task-driven concepts
(FGVC).} {Left:} per-concept residual-write projection
$\|W_{dec}[:,j]^{\top}\Delta W\|$ under a plain LoRA adapter (horizontal)
against the same quantity under \cgr{} (vertical), on a linear scale.
{Blue marks the task set $\mathcal{S}_{\mathrm{tsk}}$, red the retention
set $\mathcal{S}_{\mathrm{ret}}$, and light gray all remaining concepts. Both
sets are computed on the frozen backbone, independently of either adapter, so
the same partition applies to both axes.} The diagonal marks equal write under
both methods. A concept above it receives more update under \cgr{} than under
LoRA, a concept below it less. LoRA alone does not separate the two sets
(both horizontal-axis means are $0.20$). \cgr{} lifts the task concepts above
the diagonal (from $0.20$ under LoRA to $0.26$ under \cgr{}) and drops the
retention concepts below it (from $0.20$ to $0.11$), so the update is
reallocated, not enlarged. {Right:} the four task concepts \cgr{}
concentrates its update on, named by their top-activating FGVC images.
{Bars are blue throughout because all four belong to the task set, hollow
under LoRA and solid under \cgr{}.}}
\label{fig:interp_task}
\end{figure*}

\section{Proofs and additional theory}\label{app:proofs}

\subsection{Proof of Proposition~\ref{prop:confine}}\label{app:pf-confine}

\begin{proof}
    At any stationary point  $\Delta W^q$ of $\mathcal{L}_{\mathrm{task}}(\Delta W)+\lambda_{\mathrm{conf}}^q\|P_{\mathrm{tsk}}^{\perp}\Delta W\|_F^2$, by first order differentiability conditions, we obtain
    \begin{align*}
        P_{\mathrm{tsk}}^{\perp}\Delta W^q =-\tfrac{1}{2\lambda_{\mathrm{conf}}^q}\nabla_{\!\Delta W}\mathcal{L}_{\mathrm{task}}(\Delta W^q).
    \end{align*}
    Applying $P_{\mathrm{tsk}}^{\perp}$ to both sides and recalling that $(P_{\mathrm{tsk}}^{\perp})^2 = P_{\mathrm{tsk}}^{\perp}$,
    \begin{align*}
        P_{\mathrm{tsk}}^{\perp}\Delta W^q &=-\tfrac{1}{2\lambda_{\mathrm{conf}}^q}P_{\mathrm{tsk}}^{\perp}\nabla_{\!\Delta W}\mathcal{L}_{\mathrm{task}}(\Delta W^q).
    \end{align*}
    Thus, taking norms, 
    \begin{equation*}
        \|P_{\mathrm{tsk}}^{\perp}\Delta W^q\|_F=O(1/\lambda_{\mathrm{conf}}) \|P_{\mathrm{tsk}}^{\perp}\nabla_{\!\Delta W}\mathcal{L}_{\mathrm{task}}(\Delta W^q)\|_F,
    \end{equation*}
    and the conclusion follows from the assumption $\|P_{\mathrm{tsk}}^{\perp}\nabla_{\!\Delta W}\mathcal{L}_{\mathrm{task}}(\Delta W^q)\|_F \leq C$.

\end{proof}

\subsection{Proof of Proposition~\ref{prop:insufficiency}}\label{app:pf-insuff}

\begin{proof}[Proof of the upper bound]
Write
\[
  q(x)=W_{\mathrm{ret}}h_0(x)+b_{\mathrm{ret}},
  \qquad
  \delta(x)=W_{\mathrm{ret}}\Delta W x.
\]
Since ReLU is coordinatewise $1$-Lipschitz,
\[
\begin{aligned}
  \big\|\Delta f_{\mathrm{ret}}(x)\big\|_2
  &=
  \big\|
    \mathrm{ReLU}\big(q(x)+\delta(x)\big)
    -
    \mathrm{ReLU}\big(q(x)\big)
  \big\|_2
  \\
  &\leq
  \big\|\delta(x)\big\|_2
  =
  \big\|W_{\mathrm{ret}}\Delta W x\big\|_2.
\end{aligned}
\]
Squaring and taking expectations gives
\[
\begin{aligned}
  \mathbb{E}_x\big\|\Delta f_{\mathrm{ret}}(x)\big\|_2^2
  &\leq
  \mathbb{E}_x\big\|W_{\mathrm{ret}}\Delta W x\big\|_2^2
  \\
  &=
  \mathrm{tr}\!\left(
    W_{\mathrm{ret}}\Delta W\Sigma
    \Delta W^\top W_{\mathrm{ret}}^\top
  \right).
\end{aligned}
\]
\end{proof}

\begin{proof}[Proof of the lower bound]
For every $x\in A$, the concepts in $\mathcal{J}$ remain in the active ReLU region, and hence their read-out change is exactly
\[
  \Delta f_{\mathrm{ret},\mathcal{J}}(x)
  =
  W_{\mathcal{J}}\Delta W x.
\]
Since the squared norm of the complete read-out change dominates the contribution of these coordinates on the subset $A$,
\[
\begin{aligned}
  \mathbb{E}_x\big\|\Delta f_{\mathrm{ret}}(x)\big\|_2^2
  &\geq
  \mathbb{E}_x\!\left[
    \mathbf{1}_A(x)
    \big\|W_{\mathcal{J}}\Delta W x\big\|_2^2
  \right]
  \\
  &=
  \mathrm{tr}\!\left(
    W_{\mathcal{J}}\Delta W
    \Sigma_A
    \Delta W^\top W_{\mathcal{J}}^\top
  \right).
\end{aligned}
\]
To see why exact projection does not control this term, suppose there exist $u$ and $v$ such that
\[
  P_{\mathrm{tsk}}^\perp u=0,
  \qquad
  W_{\mathcal{J}}u\neq 0,
  \qquad
  v^\top\Sigma_Av>0.
\]
For $\Delta W_t=tuv^\top$,
\[
  P_{\mathrm{tsk}}^\perp\Delta W_t=0,
\]
while
\[
\begin{aligned}
  \mathrm{tr}\!\left(
    W_{\mathcal{J}}\Delta W_t
    \Sigma_A
    \Delta W_t^\top W_{\mathcal{J}}^\top
  \right)
  &=
  t^2
  \big(v^\top\Sigma_Av\big)
  \big\|W_{\mathcal{J}}u\big\|_2^2.
\end{aligned}
\]
Thus, along any such admissible family for which the concepts in $\mathcal{J}$ remain active on $A$, the lower bound grows quadratically with the magnitude of the in-subspace update. Hence the projection constraint alone does not control retention-concept drift.
\end{proof}

\section{Small experimental exploration of relation between weight and activation concept space}\label{app:drift}
\begin{figure*}[t]\centering
\includegraphics[width=\textwidth]{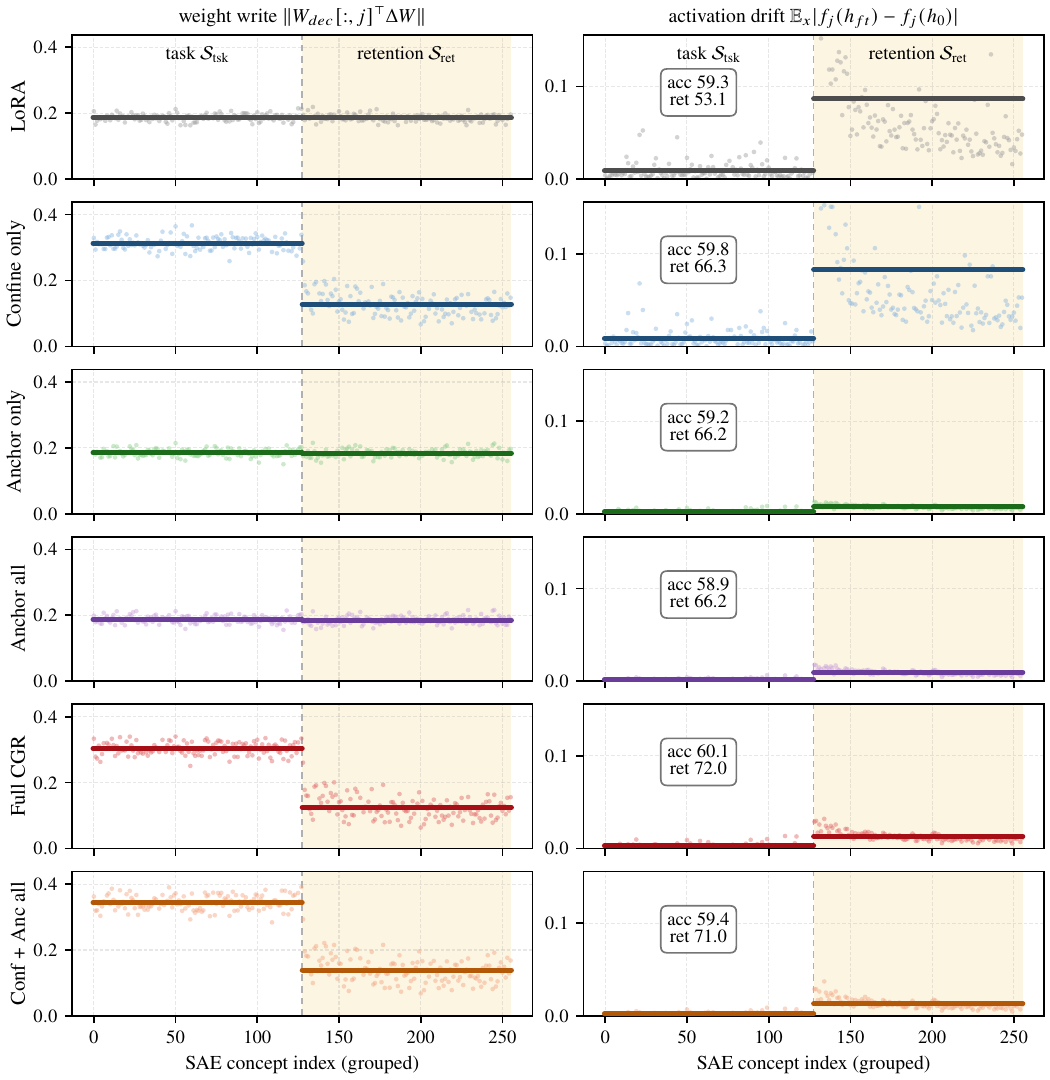}
\caption{Weight-space control and activation-space drift across task and retention concepts. Each point corresponds to one SAE concept, grouped into the task set $\mathcal{S}_{\mathrm{tsk}}$ and the retention set $\mathcal{S}_{\mathrm{ret}}$. The left column reports the weight change magnitude, $\|W_{\mathrm{dec}}[:,j]^{\top}\Delta W\|$. The right column reports the corresponding activation drift, $\mathbb{E}_{x}|f_j(h_{\mathrm{ft}}(x))-f_j(h_0(x))|$. Horizontal lines denote per-set means, and the boxes report task accuracy and CIFAR-100 zero-shot retention. \textsc{Confine} concentrates the weight update on the task concepts and suppresses its projection onto the retention concepts, but does not by itself prevent retention-concept activation drift. \textsc{Anchor} directly suppresses activation drift, including when applied only to the selected retention set. Combining the two terms in full CGR separates their roles: task concepts remain free to support adaptation, while the retention concepts constrain the resulting representation change. Anchoring all concepts instead reduces this flexibility and yields lower task accuracy than selectively anchoring $\mathcal{S}_{\mathrm{ret}}$. Results are means over three seeds; off-scale retention-concept drifts are clipped.}
\label{fig:drift_sixarm}
\end{figure*}

We want to compare the change in weight direction (in the residual stream) for each concept to their respecting concept drift in the SAE. To do that we will compare two quantities for each concept $j$, first the total weight change $\|W_{dec[:,j]}^T \Delta W\|$ and second the concept drift $\mathbb{E}_x|f_j(h_{ft})-f_j(h_0)|$. 

In Figure~\ref{fig:drift_sixarm}, we divide the concepts into task and retention and plot the weight change and concept drift for several strategies. We do that for the LoRA baseline and then add five other options: Confine only(in task set), Anchor only(in ret set), Anchor all(task and ret), CGR(Anchor ret and confine task), and Anchor all(task and ret sets) and Confine(task set). 
In terms of task accuracy and retention, the full CGR method is the best option. However, the interesting part is how different strategies control the weights and concept drifts. 

The Anchor term seems to reduce the concept drift not only in the portion it is applied to but even when applied only to the retention concepts; it helps reduce the drift in the task concepts as well. The confine term seems to always accelerate the change in weights on the task set and reduce the weight change in the retention set; nevertheless, it is unable to control the concept drift in retention. This indicates that both terms are needed. Moreover, anchoring every concept seems to be working against the new task accuracy when compared with only anchoring the retention set, both with or without confine. In an informal way, one could think that the retention concepts are not only helping in retention but also act as sanity checks; the machine changes the internal representations in weight space for the new task, but every time it does it has to pass these checks. This, at the same time, keeps the main concept drift low. For example, if we want to improve our understanding of an elephant, we can change our representation of an elephant, but it still has to be an animal and have four legs. 

\begin{figure*}[t]\centering
\includegraphics[width=\textwidth]{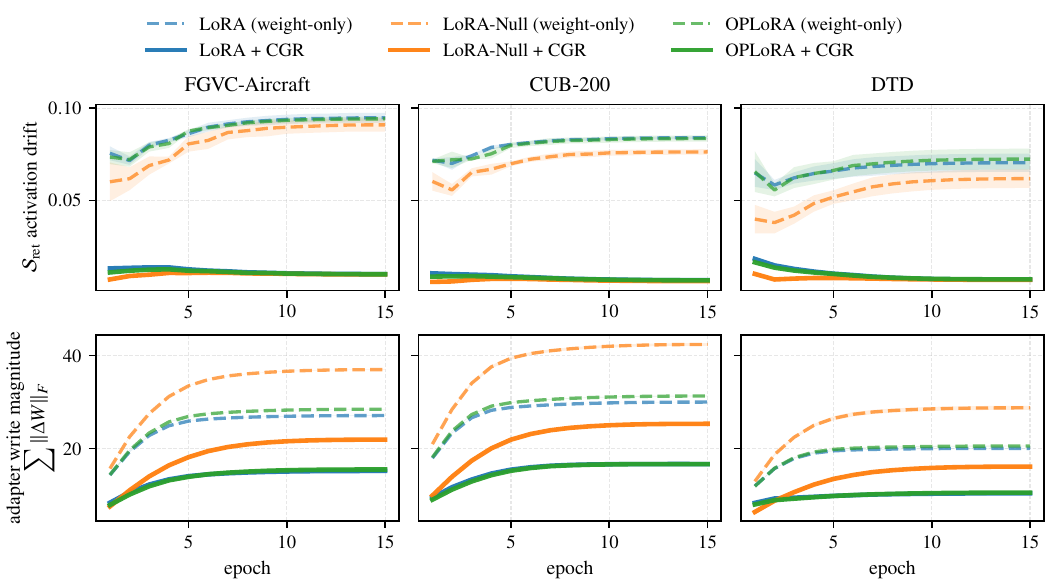}
\caption{{Drift over training: full \cgr{} pins the activation channel
and confines the weight write on all three datasets.} Columns: FGVC-Aircraft,
CUB-200, DTD ($3$ bases $\times$ weight-only dashed / $+\cgr$ solid,
mean$\pm$std over $3$ seeds, rows share y-limits). Top row: every weight-only
method's retention-concept activation drift climbs (to ${\sim}0.06$--$0.09$)
while all $+\cgr$ arms stay pinned at ${\sim}0.01$. Bottom row: the
\textsc{confine} term also bounds the adapter write magnitude on every base
(e.g.\ LoRA-Null on FGVC $37\!\to\!22$), so the drift gap is not explained by
larger weight moves.}\label{fig:drift_curve}
\end{figure*}

\iffalse
\begin{table}[t]\centering\small\tabcolsep=5pt

\begin{tabular}{l cc cc}\toprule
Base & drift WO & drift $+\cgr$ & write-mag WO$\to+\cgr$ & perp-frac \\\midrule
LoRA       & 0.095 & 0.010 & $27.1\to15.2$ & 0.089 \\
OPLoRA     & 0.094 & 0.010 & $28.5\to15.5$ & 0.054 \\
LoRA-Null  & 0.091 & 0.010 & $37.0\to21.9$ & 0.040 \\
\bottomrule\end{tabular}
\caption{\textbf{\cgr{} on every base (FGVC).}
$+\cgr$ (anchor $+$ confine) cuts the retention-concept activation drift
$\overline{D_j}_{\,j\in\mathcal{S}_{ret}}$ to ${\sim}0.01$ \emph{and} confines the
weight write on every base: the total adapter write magnitude roughly halves and
the perp-fraction (the share of $\Delta W$ \emph{outside} the task-concept
subspace) drops below $0.09$ ($>91\%$ inside). The two terms act on different
quantities (the \textsc{anchor} term pins the activation drift; the \textsc{confine}
term confines the weight write), and \cgr{} engages both on all
bases.}\label{tab:app_drift}
\end{table}
\fi

\section{Activation drift correlates with knowledge retention}\label{app:drift_ret}

The anchor's value is the retention concepts' expected squared drift, so it is worth checking that this drift moves with retention itself. It does on all three drift-instrumented datasets. Figure~\ref{fig:drift_pred} in the main text plots the FGVC relation, where retention-concept activation drift correlates with zero-shot retention at $r{=}{-}0.63$ ($p{=}8{\times}10^{-7}$) over every run logging both quantities. Figure~\ref{fig:drift_ret} replicates it on CUB-200 ($r{=}{-}0.70$, $p{=}2{\times}10^{-7}$) and DTD ($r{=}{-}0.87$, $p{=}3{\times}10^{-8}$). Retention is the CIFAR-10/100 mean, the two Table~\ref{tab:clip_vtab} probes every run logs. In each panel, the anchored runs sit at low drift and high retention, the weight-only runs at the opposite extreme. That drift is harmful is therefore an empirical finding rather than an assumption.

\begin{figure}[t]\centering
\includegraphics[width=0.9\columnwidth]{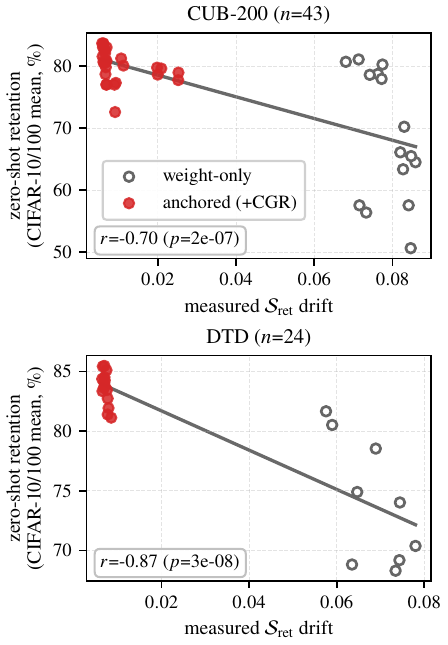}
\caption{{Drift predicts forgetting: replication (CUB-200, DTD).}
Per-run measured $\mathcal{S}_{\mathrm{ret}}$ activation drift vs.\
zero-shot retention (CIFAR-10/100 mean), for every drift-instrumented run per
dataset. Anchored runs are red, weight-only runs are open. The negative
relation from Fig.~\ref{fig:drift_pred} holds on both datasets.}\label{fig:drift_ret}
\end{figure}

\section{Causal validation of the selection}\label{app:ablation}

The previous section shows the selected concepts correlate with retention. Here we test whether they carry it. In each trained model, we zero the retention concepts $\mathcal{S}_{\mathrm{ret}}$ at the read-out layer and re-measure zero-shot retention. The control repeats the deletion on a size-matched random set, drawn only from concepts that fire on the evaluation data, since removing inactive dictionary entries would leave retention unchanged by construction. Removing the selected set destroys a large share of retained zero-shot accuracy, while removing the random set leaves retention unchanged (Fig.~\ref{fig:ablation}). A paired Wilcoxon signed-rank test over all FGVC run-probe pairs separates the two conditions at $p=1.7\times10^{-13}$, and the same gap holds on CUB-200 and DTD. Retained capability therefore flows through the concepts the \textsc{anchor} holds, so the drift it measures belongs to the mechanism rather than to a correlated quantity.

\begin{figure}[t]\centering
\includegraphics[width=0.9\columnwidth]{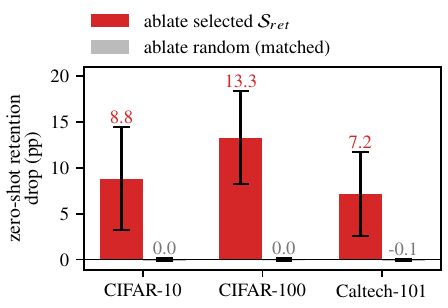}
\caption{{Selected concepts are causally load-bearing.} In each trained
model we zero the $|\mathcal{S}_{ret}|=128$ selected retention concepts at
layer~8, subtracting their SAE contribution from the residual stream with no
retraining, and re-measure zero-shot retention. The control applies the same
surgery to $128$ random alive concepts. Bars show retention drop per natural probe
(mean$\pm$std over $24$ FGVC runs, with the intervention instrumenting three of the
four retention probes and SUN397 excluded). The random set is inert. Removing
$\mathcal{S}_{ret}$ destroys natural-image retention.}\label{fig:ablation}
\end{figure}

\section{Named weight-channel concepts}\label{app:interp_weight}
Figure~\ref{fig:interp} in the main text shows which concepts \cgr{} pins in the activation channel. Figure~\ref{fig:interp_task} is its weight-channel counterpart, answering where the update goes by measurement rather than by the interventions of App.~\ref{app:ablation}. For every concept, it plots the weight write $|\Delta W \cdot W_{dec}^{(j)}|$ under a plain LoRA adapter against the same quantity under \cgr{}. Plain LoRA writes to task and retention concepts with equal weight, leaving both sets near the diagonal at a mean of $0.20$. \cgr{} lifts the task concepts to $0.26$ and drops the retention concepts to $0.11$, so the update is reallocated rather than enlarged. The right panel names the four task concepts carrying the most update by their top-activating images. All four are coherent aircraft types (military jets, airliners, propeller light planes, business jets), the fine-grained structure the classifier must separate. The update lands on the task concepts while the retention concepts hold their activations.

\paragraph{Editing the concepts at test time}
The editing test of \S\ref{subsec:interp} asks whether the selected concepts work as an interface and not only as a penalty. A concept bottleneck classifies through an interpretable concept layer, so individual concepts can be corrected at test time. We train a logistic-regression head directly on the frozen model's $256$ selected concept activations, following the concept-bottleneck configuration of \citet{koh2020cbm}. The head has no trainable projection, so no recombination can wash out the correspondence between a coordinate and its named concept. The control trains a matching head on a size-matched random set, drawn from concepts that fire on the evaluation data as in \S\ref{app:ablation}. Figure~\ref{fig:cgr_cbm_edit} shows both arms on all three drift-instrumented datasets. Before any editing, the selected head reaches $10.7\%$ on FGVC-Aircraft, $7.3\%$ on CUB-200 and $40.5\%$ on DTD, against $1.5\%$, $1.3\%$ and $8.7\%$ for the control, four to seven times the control on every dataset. The selected concepts therefore carry class-relevant structure that an equally large random slice of the dictionary does not. At test time we intervene on a fraction of the selected concepts, replacing each with its ground-truth class-conditional mean, and re-predict, the standard evaluation protocol for bottleneck models~\citep{koh2020cbm}. We intervene on both heads identically, so the contrast isolates the concept set rather than the information the intervention supplies. Accuracy rises steadily as the edited fraction grows, reaching $43.9\%$, $37.2\%$ and $97.9\%$ under full correction, and the selected head leads the control at every fraction on every dataset. Writing correct values into \emph{named} concepts repairs wrong predictions, which is the intervenability property a concept bottleneck is meant to provide.

\begin{table*}[t]\centering\small\tabcolsep=5pt

\definecolor{cgrband}{gray}{0.90}
\begin{tabular}{l cccc c}\toprule
Method & CIFAR-10 & CIFAR-100 & Caltech-101 & SUN397 & {Avg} \\\midrule
Zero-shot (no FT) & 93.8 & 78.2 & 90.9 & 67.6 & 82.6 \\
\midrule
LoRA & $84.3$\,{$\pm4.1$} & $63.1$\,{$\pm5.8$} & $87.3$\,{$\pm2.5$} & $61.6$\,{$\pm2.7$} & $74.1$\,{$\pm3.5$} \\
\rowcolor{cgrband}
\quad ${+}$\cgr{}$^*$ & $91.4$\,{$\pm1.8$} & $74.0$\,{$\pm2.2$} & $89.5$\,{$\pm1.9$} & $66.0$\,{$\pm1.1$} & $80.2$\,{$\pm1.5$} \\
\addlinespace
OPLoRA & $84.5$\,{$\pm3.4$} & $63.3$\,{$\pm5.5$} & $87.6$\,{$\pm2.3$} & $61.9$\,{$\pm2.7$} & $74.4$\,{$\pm3.3$} \\
\rowcolor{cgrband}
\quad ${+}$\cgr{}$^*$ & $92.0$\,{$\pm1.2$} & $75.2$\,{$\pm1.4$} & $90.0$\,{$\pm1.2$} & $66.2$\,{$\pm1.2$} & $80.8$\,{$\pm1.1$} \\
\addlinespace
LoRA-Null & $89.9$\,{$\pm2.5$} & $72.5$\,{$\pm4.6$} & $89.6$\,{$\pm1.9$} & $66.5$\,{$\pm1.3$} & $79.6$\,{$\pm2.0$} \\
\rowcolor{cgrband}
\quad ${+}$\cgr{}$^*$ & $\mathbf{92.6}$\,{$\pm1.3$} & $76.8$\,{$\pm1.0$} & $\mathbf{90.2}$\,{$\pm1.7$} & $\mathbf{67.5}$\,{$\pm0.5$} & $\mathbf{81.7}$\,{$\pm0.7$} \\
\addlinespace
CorDA & $87.0$\,{$\pm5.0$} & $67.9$\,{$\pm6.9$} & $88.2$\,{$\pm3.1$} & $63.4$\,{$\pm4.7$} & $76.6$\,{$\pm4.6$} \\
\rowcolor{cgrband}
\quad ${+}$\cgr{}$^*$ & $91.4$\,{$\pm2.4$} & $73.2$\,{$\pm3.1$} & $89.4$\,{$\pm2.1$} & $65.3$\,{$\pm2.9$} & $79.8$\,{$\pm2.5$} \\
\addlinespace
AdaptFormer & $89.4$\,{$\pm3.2$} & $72.5$\,{$\pm2.7$} & $88.2$\,{$\pm3.0$} & $64.4$\,{$\pm2.2$} & $78.6$\,{$\pm1.9$} \\
\rowcolor{cgrband}
\quad ${+}$\cgr{}$^*$ & $92.2$\,{$\pm1.9$} & $\mathbf{77.0}$\,{$\pm0.6$} & $89.8$\,{$\pm2.3$} & $66.7$\,{$\pm0.8$} & $81.4$\,{$\pm1.0$} \\
\addlinespace
SAILS-init & $82.6$\,{$\pm4.3$} & $60.8$\,{$\pm5.6$} & $86.9$\,{$\pm2.8$} & $60.6$\,{$\pm2.7$} & $72.7$\,{$\pm3.6$} \\
\rowcolor{cgrband}
\quad ${+}$\cgr{}$^*$ & $91.0$\,{$\pm1.9$} & $73.5$\,{$\pm2.0$} & $89.4$\,{$\pm1.6$} & $65.4$\,{$\pm1.3$} & $79.8$\,{$\pm1.4$} \\
\bottomrule\end{tabular}
\caption{{Per-probe zero-shot retention (\%).} The four natural probes that make up the {Ret} column of Table~\ref{tab:clip_vtab}, each shown separately; every cell is meaned over the same $10$ fine-tuning tasks and $3$ seeds, so each row's {Avg} matches its Avg retention in Table~\ref{tab:clip_vtab}. Each cell shows mean${\pm}$std, the spread {across the $10$ tasks} (as in Table~\ref{tab:clip_vtab}; seed noise ${\le}0.9$); \cgr{} tightens it on every probe. Best per column {bold}; $^*\!=\!$ours.}
\label{tab:probe_retention}
\end{table*}

\begin{figure}[t]\centering
\includegraphics[width=0.9\columnwidth]{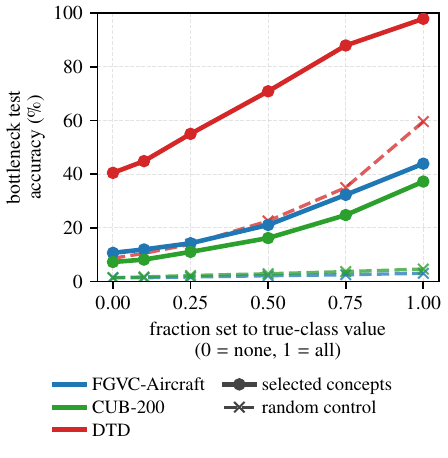}
\caption{The selected head leads the control at every level. Both heads see only concept values, so these accuracies are not comparable to the fine-tuned models of Table~\ref{tab:clip_vtab}.}
\label{fig:cgr_cbm_edit}
\end{figure}

Steering generalizes the removal test to a continuous scale factor
(Fig.~\ref{fig:steer}). On the frozen backbone, we scale the selected concepts'
contribution at the read-out layer by $\alpha$, so $\alpha{=}1$ is the
unmodified model and $\alpha{=}0$ is the removal test above, and we re-measure
zero-shot retention. Retention follows $\alpha$ smoothly and peaks at the
unmodified model. It falls monotonically toward removal, from $86.0\%$ at
$\alpha{=}1$ to $80.5\%$ at $\alpha{=}0$, and it falls again when the
contribution is amplified, reaching $85.0\%$ at $\alpha{=}1.5$. The
size-matched random control varies by $0.2$pp across the whole range.
Retention responds to the selected concepts as a dial rather than a switch,
and it responds only to them.

\begin{figure}[t]\centering
\includegraphics[width=0.9\columnwidth]{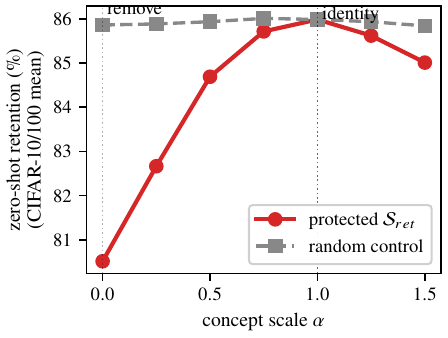}
\caption{{Steering the selected concepts controls retention (FGVC).}
On the frozen backbone we scale the selected $\mathcal{S}_{\mathrm{ret}}$
concepts' layer-8 contribution by $\alpha$ (red). Retention peaks at the
unmodified model ($\alpha{=}1$) and falls on both sides, toward removal at
$\alpha{=}0$ and under amplification at $\alpha{=}1.5$. A size-matched random
control is flat (gray). Retention is averaged over the CIFAR-10 and CIFAR-100
probes.}
\label{fig:steer}
\end{figure}

\section{Detailed vision results}\label{app:probes}\label{app:fullpanel}

In this section, we show a detailed breakdown of the results in Table~\ref{tab:clip_vtab}, both over the same {$10$} fine-tuning tasks and $3$ seeds. Table~\ref{tab:probe_retention} shows its
four-probe retention column (\S\ref{subsec:clip}) into the individual probes.
Table~\ref{tab:clip_vtab_full} reports every cell as mean${\pm}$std over the
seeds, with average accuracy and average retention across the ten tasks.

\begin{sidewaystable*}\centering\small\tabcolsep=1.5pt
\definecolor{cgrband}{gray}{0.90}
\begin{tabular}{l cc cc cc cc cc}\toprule
& \multicolumn{2}{c}{FGVC} & \multicolumn{2}{c}{CUB} & \multicolumn{2}{c}{DTD} & \multicolumn{2}{c}{Flow} & \multicolumn{2}{c}{Pets} \\
\cmidrule(lr){2-3}\cmidrule(lr){4-5}\cmidrule(lr){6-7}\cmidrule(lr){8-9}\cmidrule(lr){10-11}
Method & Acc & Ret & Acc & Ret & Acc & Ret & Acc & Ret & Acc & Ret \\\midrule
Zero-shot
  & 19.3 & 82.6 & 70.9 & 82.6 & 54.6 & 82.6 & 73.0 & 82.6 & 89.8 & 82.6 \\
\midrule
LoRA
  & $59.3{\pm}1.7$ & $67.9{\pm}1.7$ & $74.6{\pm}0.4$ & $70.1{\pm}1.1$ & $71.5{\pm}1.7$ & $74.1{\pm}0.1$ & $93.2{\pm}0.3$ & $76.3{\pm}1.8$ & $82.6{\pm}0.8$ & $73.8{\pm}2.2$ \\
\rowcolor{cgrband}
\;${+}$\cgr{}
  & $60.1{\pm}0.4$ & $78.0{\pm}0.7$ & $76.4{\pm}0.3$ & $78.7{\pm}0.4$ & $74.3{\pm}0.7$ & $80.5{\pm}0.5$ & $95.0{\pm}0.5$ & $81.6{\pm}0.4$ & $86.1{\pm}0.5$ & $80.8{\pm}0.1$ \\
OPLoRA
  & $58.9{\pm}0.4$ & $69.9{\pm}1.1$ & $75.1{\pm}0.2$ & $69.3{\pm}1.6$ & $72.1{\pm}0.5$ & $72.1{\pm}1.1$ & $93.4{\pm}0.4$ & $76.9{\pm}1.7$ & $81.5{\pm}2.4$ & $75.9{\pm}2.4$ \\
\rowcolor{cgrband}
\;${+}$\cgr{}
  & $59.5{\pm}0.3$ & $79.2{\pm}0.3$ & $76.6{\pm}0.1$ & $79.2{\pm}0.7$ & $74.6{\pm}1.0$ & $80.5{\pm}0.2$ & $95.4{\pm}0.4$ & $81.6{\pm}0.5$ & $85.7{\pm}1.4$ & $81.1{\pm}0.3$ \\
LoRA-Null
  & $61.0{\pm}0.3$ & $78.8{\pm}0.5$ & $76.2{\pm}0.7$ & $78.0{\pm}1.3$ & $73.9{\pm}1.2$ & $76.2{\pm}1.5$ & $94.8{\pm}0.2$ & $81.5{\pm}0.2$ & $86.5{\pm}1.0$ & $80.6{\pm}0.8$ \\
\rowcolor{cgrband}
\;${+}$\cgr{}
  & $60.5{\pm}0.2$ & $\mathbf{80.5}{\pm}0.3$ & $76.9{\pm}0.5$ & $\mathbf{80.5}{\pm}0.3$ & $76.7{\pm}0.6$ & $\mathbf{81.7}{\pm}0.1$ & $95.9{\pm}0.3$ & $\mathbf{82.6}{\pm}0.1$ & $87.1{\pm}1.1$ & $\mathbf{82.0}{\pm}0.2$ \\
CorDA
  & $57.8{\pm}1.1$ & $71.7{\pm}0.6$ & $72.5{\pm}0.9$ & $65.9{\pm}1.0$ & $73.3{\pm}0.7$ & $78.8{\pm}0.6$ & $93.9{\pm}0.4$ & $80.4{\pm}0.4$ & $84.4{\pm}0.8$ & $78.8{\pm}1.6$ \\
\rowcolor{cgrband}
\;${+}$\cgr{}
  & $57.1{\pm}1.1$ & $76.3{\pm}0.5$ & $73.0{\pm}0.6$ & $74.4{\pm}0.9$ & $73.9{\pm}0.5$ & $81.1{\pm}0.5$ & $94.4{\pm}0.6$ & $81.8{\pm}0.2$ & $85.5{\pm}0.9$ & $80.5{\pm}0.4$ \\
AF
  & $60.9{\pm}1.2$ & $77.1{\pm}1.2$ & $77.1{\pm}0.6$ & $75.9{\pm}0.7$ & $74.7{\pm}1.2$ & $76.0{\pm}0.9$ & $95.5{\pm}0.2$ & $80.8{\pm}0.4$ & $87.4{\pm}0.5$ & $77.4{\pm}0.7$ \\
\rowcolor{cgrband}
\;${+}$\cgr{}
  & $61.2{\pm}0.4$ & $79.5{\pm}0.2$ & $77.7{\pm}0.2$ & $79.8{\pm}0.0$ & $76.9{\pm}0.1$ & $81.3{\pm}0.4$ & $96.3{\pm}0.1$ & $82.1{\pm}0.0$ & $88.1{\pm}1.0$ & $81.5{\pm}0.1$ \\
SAILS
  & $59.0{\pm}1.5$ & $65.8{\pm}2.8$ & $74.4{\pm}0.3$ & $69.7{\pm}1.6$ & $71.0{\pm}0.3$ & $72.4{\pm}0.7$ & $92.9{\pm}0.4$ & $74.6{\pm}0.7$ & $82.7{\pm}0.8$ & $73.0{\pm}2.3$ \\
\rowcolor{cgrband}
\;${+}$\cgr{}
  & $59.0{\pm}1.1$ & $77.9{\pm}0.8$ & $76.3{\pm}0.5$ & $78.3{\pm}0.5$ & $74.3{\pm}0.3$ & $79.7{\pm}0.5$ & $95.2{\pm}0.5$ & $81.3{\pm}0.1$ & $84.5{\pm}1.6$ & $80.3{\pm}0.5$ \\
\bottomrule\end{tabular}
\par\vspace{5pt}
\begin{tabular}{l cc cc cc cc cc | cc}\toprule
& \multicolumn{2}{c}{EuroS} & \multicolumn{2}{c}{RESISC} & \multicolumn{2}{c}{DiabR} & \multicolumn{2}{c}{dSpr-O} & \multicolumn{2}{c}{SVHN} & \multicolumn{2}{c}{\textbf{Avg}} \\
\cmidrule(lr){2-3}\cmidrule(lr){4-5}\cmidrule(lr){6-7}\cmidrule(lr){8-9}\cmidrule(lr){10-11}\cmidrule(lr){12-13}
Method & Acc & Ret & Acc & Ret & Acc & Ret & Acc & Ret & Acc & Ret & Acc & Ret \\\midrule
Zero-shot
  & 40.5 & 82.6 & 61.9 & 82.6 & 35.4 & 82.6 & 9.3 & 82.6 & 60.4 & 82.6 & $51.5{\pm}25.2$ & 82.6 \\
\midrule
LoRA
  & $93.6{\pm}1.5$ & $78.3{\pm}1.0$ & $83.0{\pm}0.3$ & $76.7{\pm}0.4$ & $72.3{\pm}1.4$ & $78.6{\pm}5.0$ & $43.6{\pm}1.6$ & $70.9{\pm}4.4$ & $89.4{\pm}1.1$ & $73.8{\pm}0.9$ & $76.3{\pm}15.8$ & $74.1{\pm}3.5$ \\
\rowcolor{cgrband}
\;${+}$\cgr{}
  & $93.4{\pm}0.8$ & $81.2{\pm}0.4$ & $84.0{\pm}0.5$ & $81.8{\pm}0.3$ & $72.9{\pm}0.6$ & $81.2{\pm}1.4$ & $42.9{\pm}1.7$ & $77.9{\pm}0.6$ & $85.8{\pm}1.0$ & $80.6{\pm}0.7$ & $77.1{\pm}15.9$ & $80.2{\pm}1.5$$^{\star}$ \\
OPLoRA
  & $95.1{\pm}1.7$ & $78.9{\pm}1.3$ & $82.7{\pm}1.4$ & $77.3{\pm}1.1$ & $72.6{\pm}0.2$ & $76.4{\pm}6.5$ & $44.2{\pm}1.6$ & $72.2{\pm}2.2$ & $89.8{\pm}1.0$ & $74.6{\pm}0.9$ & $76.5{\pm}15.8$ & $74.4{\pm}3.3$ \\
\rowcolor{cgrband}
\;${+}$\cgr{}
  & $93.4{\pm}0.5$ & $81.4{\pm}0.6$ & $84.3{\pm}1.2$ & $82.1{\pm}0.2$ & $73.3{\pm}0.8$ & $81.7{\pm}0.1$ & $43.4{\pm}0.8$ & $80.0{\pm}0.3$ & $86.3{\pm}0.7$ & $81.5{\pm}0.1$ & $77.2{\pm}15.9$ & $80.8{\pm}1.1$$^{\star}$ \\
LoRA-Null
  & $95.6{\pm}0.7$ & $81.2{\pm}0.4$ & $85.0{\pm}1.2$ & $80.6{\pm}0.2$ & $72.5{\pm}0.6$ & $81.8{\pm}1.4$ & $44.6{\pm}2.0$ & $81.0{\pm}0.7$ & $91.3{\pm}0.7$ & $76.7{\pm}1.6$ & $78.1{\pm}16.1$ & $79.6{\pm}2.0$ \\
\rowcolor{cgrband}
\;${+}$\cgr{}
  & $93.3{\pm}0.2$ & $\mathbf{82.3}{\pm}0.2$ & $85.0{\pm}0.1$ & $\mathbf{82.5}{\pm}0.3$ & $73.5{\pm}0.6$ & $82.0{\pm}0.4$ & $43.9{\pm}2.0$ & $\mathbf{81.7}{\pm}0.2$ & $87.9{\pm}0.3$ & $81.7{\pm}0.2$ & $78.1{\pm}15.9$ & $\mathbf{81.7}{\pm}0.7$$^{\star}$ \\
CorDA
  & $94.5{\pm}1.1$ & $79.1{\pm}1.0$ & $84.6{\pm}0.5$ & $78.9{\pm}0.5$ & $73.2{\pm}0.5$ & $79.5{\pm}3.1$ & $43.1{\pm}2.9$ & $78.8{\pm}0.8$ & $90.1{\pm}1.0$ & $74.4{\pm}2.4$ & $76.7{\pm}16.5$ & $76.6{\pm}4.6$ \\
\rowcolor{cgrband}
\;${+}$\cgr{}
  & $94.3{\pm}0.4$ & $81.5{\pm}0.2$ & $84.4{\pm}1.0$ & $81.3{\pm}0.1$ & $72.3{\pm}0.9$ & $81.1{\pm}0.9$ & $41.9{\pm}0.4$ & $79.9{\pm}0.6$ & $87.4{\pm}0.9$ & $80.4{\pm}0.4$ & $76.4{\pm}16.7$ & $79.8{\pm}2.5$$^{\star}$ \\
AF
  & $95.5{\pm}0.5$ & $80.4{\pm}0.4$ & $85.7{\pm}0.8$ & $80.0{\pm}0.3$ & $72.8{\pm}0.8$ & $80.5{\pm}0.6$ & $45.1{\pm}1.9$ & $79.8{\pm}0.1$ & $90.7{\pm}0.1$ & $78.1{\pm}1.6$ & $78.5{\pm}16.1$ & $78.6{\pm}1.9$ \\
\rowcolor{cgrband}
\;${+}$\cgr{}
  & $94.4{\pm}0.9$ & $81.9{\pm}0.2$ & $85.9{\pm}0.3$ & $82.3{\pm}0.1$ & $73.5{\pm}0.6$ & $\mathbf{82.5}{\pm}0.2$ & $44.8{\pm}1.3$ & $81.4{\pm}0.3$ & $86.6{\pm}0.8$ & $\mathbf{82.0}{\pm}0.1$ & $78.5{\pm}15.8$ & $81.4{\pm}1.0$$^{\star}$ \\
SAILS
  & $94.5{\pm}2.0$ & $76.2{\pm}1.9$ & $83.2{\pm}1.7$ & $75.2{\pm}1.1$ & $72.6{\pm}0.8$ & $77.7{\pm}6.5$ & $43.8{\pm}2.3$ & $69.0{\pm}0.4$ & $90.0{\pm}0.9$ & $73.7{\pm}2.5$ & $76.4{\pm}15.9$ & $72.7{\pm}3.6$ \\
\rowcolor{cgrband}
\;${+}$\cgr{}
  & $93.1{\pm}0.8$ & $81.2{\pm}0.3$ & $84.0{\pm}1.5$ & $81.4{\pm}0.3$ & $72.1{\pm}1.6$ & $80.0{\pm}0.3$ & $42.0{\pm}2.0$ & $77.7{\pm}0.9$ & $86.0{\pm}0.7$ & $80.7{\pm}0.4$ & $76.7{\pm}16.2$ & $79.8{\pm}1.4$$^{\star}$ \\
\bottomrule\end{tabular}
\caption{{{Complete per-task results of \cgr{} on the ten CLIP ViT-B/32 transfer tasks (extended version of Table~\ref{tab:clip_vtab}).} Every cell is mean${\pm}$std over $3$ seeds. {Acc} is target-task accuracy; {Ret} is the four-probe zero-shot retention of Table~\ref{tab:clip_vtab}. \textbf{Avg} is the mean over the same $10$ tasks (${\pm}$std across tasks), so the Avg Ret column matches Table~\ref{tab:clip_vtab}. Best retention per task and best Avg Ret in {bold}; the star ($^{\star}$) marks a significant across-task \cgr{} retention gain (paired Wilcoxon, $p{<}0.05$).}}
\label{tab:clip_vtab_full}
\end{sidewaystable*}

\section{Detailed language results}\label{app:llmstd}
Table~\ref{tab:llm_instruct_std} reports the seed variability behind
Table~\ref{tab:llm_instruct}. Every retention column is the mean and sample
standard deviation over three seeds. The \cgr{} retention gains clear the
seed spread on every backbone.
\begin{table*}[t]\centering\small\tabcolsep=5pt
\definecolor{cgrband}{gray}{0.90}
\begin{tabular}{l ccccc}\toprule
 & \multicolumn{5}{c}{World-knowledge retention (\%)} \\
\cmidrule(lr){2-6}
Method & Trivia & NQ-O & WebQS & Avg1 & Per \\\midrule
Zero-shot & $61.5$ & $7.9$ & $11.6$ & $27.0$ & $100.0$ \\
\midrule
LoRA & $27.6$\,{$\pm1.2$} & $11.1$\,{$\pm0.6$} & $12.4$\,{$\pm0.7$} & $17.1$\,{$\pm0.7$} & $81.6$\,{$\pm0.7$} \\
\rowcolor{cgrband}
\quad +\cgr{} & $43.5$\,{$\pm2.2$} & $12.0$\,{$\pm0.8$} & $12.4$\,{$\pm0.7$} & $22.6$\,{$\pm0.4$} & $90.2$\,{$\pm1.1$} \\
LoRA-Null & $47.1$\,{$\pm0.3$} & $9.5$\,{$\pm0.3$} & $9.8$\,{$\pm0.1$} & $22.1$\,{$\pm0.1$} & $87.1$\,{$\pm0.3$} \\
\rowcolor{cgrband}
\quad +\cgr{} & $55.9$\,{$\pm0.3$} & $9.1$\,{$\pm0.2$} & $9.7$\,{$\pm0.2$} & $24.9$\,{$\pm0.1$} & $91.5$\,{$\pm0.8$} \\
CorDA & $40.5$\,{$\pm2.7$} & $12.1$\,{$\pm1.2$} & $10.8$\,{$\pm1.7$} & $21.1$\,{$\pm1.8$} & $85.1$\,{$\pm4.5$} \\
\rowcolor{cgrband}
\quad +\cgr{} & $45.7$\,{$\pm1.2$} & $13.2$\,{$\pm0.5$} & $9.4$\,{$\pm0.2$} & $22.8$\,{$\pm0.3$} & $85.0$\,{$\pm0.5$} \\
OPLoRA & $40.2$\,{$\pm1.5$} & $10.2$\,{$\pm0.6$} & $11.5$\,{$\pm0.7$} & $20.6$\,{$\pm0.4$} & $87.6$\,{$\pm1.4$} \\
\rowcolor{cgrband}
\quad +\cgr{} & $51.1$\,{$\pm1.3$} & $9.4$\,{$\pm0.8$} & $11.4$\,{$\pm0.3$} & $24.0$\,{$\pm0.8$} & $93.8$\,{$\pm1.5$} \\
SAILS & $32.0$\,{$\pm3.3$} & $11.2$\,{$\pm0.8$} & $12.2$\,{$\pm0.3$} & $18.5$\,{$\pm1.2$} & $84.0$\,{$\pm1.8$} \\
\rowcolor{cgrband}
\quad +\cgr{} & $42.9$\,{$\pm3.1$} & $11.6$\,{$\pm1.1$} & $12.1$\,{$\pm0.4$} & $22.2$\,{$\pm0.8$} & $89.9$\,{$\pm1.7$} \\
AdaptFormer & $38.1$\,{$\pm0.3$} & $3.7$\,{$\pm0.2$} & $8.2$\,{$\pm0.1$} & $16.7$\,{$\pm0.1$} & $59.6$\,{$\pm1.1$} \\
\rowcolor{cgrband}
\quad +\cgr{} & $59.4$\,{$\pm0.1$} & $6.2$\,{$\pm0.1$} & $9.9$\,{$\pm0.1$} & $25.1$\,{$\pm0.1$} & $86.7$\,{$\pm0.4$} \\
\bottomrule\end{tabular}
\caption{Seed variability of the Llama-3.1-8B retention results. Every cell
of Table~\ref{tab:llm_instruct} is reported as mean\,$\pm$\, sample std over
seeds. {Per} is capped per seed here, whereas Table~\ref{tab:llm_instruct} applies
the cap to the seed mean, which shifts {Per} by at most $1.0$ where a seed
exceeds the zero-shot base. MTBench is single-seed by design and omitted.}
\label{tab:llm_instruct_std}
\end{table*}

\section{Anchor and confine robustness}\label{app:lambda}

Retention is insensitive to the anchor coefficient $\lambda_{\mathrm{anc}}$ over a wide range. Table~\ref{tab:lambda_cvar} sweeps it $15\times$ on both fine-grained datasets. Accuracy stays within seed variation while retention rises monotonically and plateaus above $\lambda_{\mathrm{anc}}{=}100$, from $69.0$ to $73.8$ on FGVC and from $71.2$ to $74.0$ on CUB-200. FGVC is still rising at the top of the range; CUB-200 has settled, its last two settings differing by less than one standard deviation. The LoRA base without the anchor retains $53.1$ and $59.5$ (Table~\ref{tab:ablation}), so most of the gain comes from the term being active at all and the response beyond that is smooth. The coefficient is a strength dial with a wide safe range rather than a constant tuned per task.

\begin{table}[t]\centering\small\tabcolsep=5pt

\begin{tabular}{l c @{\hskip 1.4em} l c}\toprule
\multicolumn{2}{c}{Natural} & \multicolumn{2}{c}{Specialized / Structured}\\
\cmidrule(r){1-2}\cmidrule(l){3-4}
Task & EV & Task & EV \\\midrule
Caltech-101    & 0.87 & RESISC-45        & 0.74 \\
\rowcolor{tabband}
SUN397         & 0.86 & EuroSAT          & 0.69 \\
DTD            & 0.85 & Diabetic-Retino  & 0.49 \\
\rowcolor{tabband}
CUB-200        & 0.82 & dSprites-loc/ori & 0.23 \\
Oxford-Pets    & 0.81 &                  &      \\
\rowcolor{tabband}
Flowers-102    & 0.76 &                  &      \\
CIFAR-100      & 0.75 &                  &      \\
\rowcolor{tabband}
CIFAR-10       & 0.72 &                  &      \\
SVHN           & 0.67 &                  &      \\
\rowcolor{tabband}
FGVC-Aircraft  & 0.66 &                  &      \\
\bottomrule\end{tabular}
\caption{{SAE reconstruction EV (layer-8 CLS, 1000 samples, frozen
backbone).} Higher $=$ SAE in-distribution, concepts faithful. Natural tasks are
high-EV; the structured/synthetic task (dSprites) is
low-EV.}\label{tab:app_ev}
\end{table}

\begin{table*}[t]\centering\small\tabcolsep=5pt

\begin{tabular}{ll ccccc}\toprule
 & $\lambda_{\mathrm{anc}}$ & 10 & 30 & 60 & 100 & 150 \\\midrule
\multirow{2}{*}{FGVC} & Acc & $59.4$\,{$\pm1.0$} & $58.8$\,{$\pm1.7$} & $60.1$\,{$\pm0.4$} & $59.8$\,{$\pm0.2$} & $59.4$\,{$\pm0.3$} \\
 & C-100 ret. & $69.0$\,{$\pm0.6$} & $70.0$\,{$\pm1.1$} & $72.0$\,{$\pm0.8$} & $73.3$\,{$\pm0.8$} & $73.8$\,{$\pm0.4$} \\\midrule
\multirow{2}{*}{CUB-200} & Acc & $76.0$\,{$\pm0.2$} & $76.3$\,{$\pm0.1$} & $76.4$\,{$\pm0.3$} & $76.3$\,{$\pm0.2$} & $76.2$\,{$\pm0.2$} \\
 & C-100 ret. & $71.2$\,{$\pm0.6$} & $72.1$\,{$\pm0.5$} & $73.2$\,{$\pm0.7$} & $74.2$\,{$\pm0.5$} & $74.0$\,{$\pm0.6$} \\
\bottomrule\end{tabular}
\caption{{Anchor-strength sweep (\cgr{} on LoRA r8).} Retention rises
with $\lambda_{\mathrm{anc}}$ and flattens above $100$ on both datasets, at
accuracy within seed noise. The default $\lambda_{\mathrm{anc}}{=}60$
reproduces the full-\cgr{} row of Table~\ref{tab:ablation}. Mean over $3$
seeds, $\pm$std.}
\label{tab:lambda_cvar}
\end{table*}

\paragraph{Confine strength.}
The same robustness holds for the confine coefficient. Sweeping $\lambda_{\mathrm{conf}}\in{0.03,0.1,0.3,1.0}$ under the Table~\ref{tab:concept_id} protocol (LoRA r8, anchor $\lambda_{\mathrm{anc}}{=}60$, attribution-selected task set, $3$ seeds) leaves both axes on a plateau on both datasets (Table~\ref{tab:lamconf}). Retention varies by $1.8$ points on FGVC-Aircraft and stays within seed variation on CUB-200 over the $33\times$ range, at flat accuracy. The term still engages. The perpendicular fraction of the update, the share of $\Delta W$ outside the task subspace, falls monotonically from $0.26$ to ${\sim}0$ as $\lambda_{\mathrm{conf}}$ grows, so the plateau reflects a wide safe range rather than an inert term. The flanking columns give the endpoints, anchor alone at $\lambda_{\mathrm{conf}}{=}0$ and a hard projection replacing the soft confine, and both retain less than every setting inside the sweep.

\begin{table*}[t]\centering\small\tabcolsep=5pt
\begin{tabular}{ll cccccc}\toprule
 & & \multirow{2}{*}{Anchor only} & \multicolumn{4}{c}{$\lambda_{\mathrm{conf}}$}
   & \multirow{2}{*}{Hard proj.\ ${+}$ anchor} \\
\cmidrule(lr){4-7}
 & & & $0.03$ & $0.1$ & $0.3$ & $1.0$ & \\\midrule
\multirow{2}{*}{FGVC}
  & Acc        & $59.2$ & $60.4$\,{$\pm0.9$} & $59.4$\,{$\pm1.0$} & $60.7$\,{$\pm1.0$} & $60.4$\,{$\pm0.7$} & $58.4$ \\
  & C-100 ret. & $66.2$ & $70.3$\,{$\pm0.6$} & $72.1$\,{$\pm0.3$} & $71.8$\,{$\pm0.5$} & $70.9$\,{$\pm0.9$} & $63.8$ \\\midrule
\multirow{2}{*}{CUB-200}
  & Acc        & $75.6$ & $76.2$\,{$\pm0.1$} & $76.0$\,{$\pm0.4$} & $76.4$\,{$\pm0.2$} & $76.3$\,{$\pm0.4$} & $75.2$ \\
  & C-100 ret. & $70.6$ & $73.0$\,{$\pm1.5$} & $72.9$\,{$\pm1.2$} & $73.2$\,{$\pm1.2$} & $73.1$\,{$\pm1.6$} & $69.4$ \\
\bottomrule\end{tabular}
\caption{{Confine-strength sweep (\cgr{} on LoRA r8).} Accuracy and CIFAR-100 zero-shot retention stay on a plateau over a $33\times$ range of $\lambda_{\mathrm{conf}}$, and the default $\lambda_{\mathrm{conf}}{=}0.1$ sits inside it. The flanking columns give the endpoints. Anchor alone is the $\lambda_{\mathrm{conf}}{=}0$ limit, and replacing the soft confine with a hard projection onto the task subspace retains less still. Every setting inside the sweep beats both.}
\label{tab:lamconf}
\end{table*}

\paragraph{Retention-set size.}
The remaining anchor hyperparameter is the retention-set size $|\mathcal{S}{\mathrm{ret}}|$. Sweeping $k\in{32,64,128,256,512}$ under the same protocol leaves both axes on a plateau on both datasets (Table~\ref{tab:ksweep}). Retention varies by $2.3$ points on FGVC and $1.8$ on CUB-200 over the $16\times$ range, declining gently toward the all-active limit of Table~\ref{tab:concept_id} ($66.6$ and $68.6$). Adding ever-lower-presence concepts dilutes the anchor uniformly rather than protecting more, so dilution is the mechanism and not anchor breadth itself. An importance-weighted all-active variant, each concept weighted by its mean activation magnitude with weights normalized to mean $1$, recovers the selected set's retention at matched $\lambda{\mathrm{anc}}$ (FGVC $73.0{\pm}0.7$, CUB-200 $74.1{\pm}0.6$ against top-$k$'s $72.0{\pm}0.8$ and $73.2{\pm}0.7$, accuracy tied) by re-concentrating the anchoring mass on high-firing concepts. Selection reaches the same concentration with $128$ named concepts instead of ${\sim}4$k weighted ones, which is what keeps the anchored set intervenable (App.~\ref{app:ablation}).

\begin{table*}[t]\centering\small\tabcolsep=5pt

\begin{tabular}{ll ccccc}\toprule
 & $|\mathcal{S}_{\mathrm{ret}}|$ & 32 & 64 & 128 & 256 & 512 \\\midrule
\multirow{2}{*}{FGVC} & Acc        & $59.4$\,{$\pm0.4$} & $60.0$\,{$\pm0.5$} & $60.1$\,{$\pm0.4$} & $59.7$\,{$\pm0.6$} & $59.8$\,{$\pm0.8$} \\
 & C-100 ret. & $72.8$\,{$\pm0.4$} & $73.2$\,{$\pm0.4$} & $72.0$\,{$\pm0.8$} & $71.8$\,{$\pm0.5$} & $70.9$\,{$\pm0.8$} \\\midrule
\multirow{2}{*}{CUB-200} & Acc     & $76.2$\,{$\pm0.6$} & $76.3$\,{$\pm0.2$} & $76.4$\,{$\pm0.3$} & $76.3$\,{$\pm0.1$} & $76.0$\,{$\pm0.1$} \\
 & C-100 ret. & $74.0$\,{$\pm0.4$} & $73.3$\,{$\pm0.8$} & $73.2$\,{$\pm0.7$} & $72.7$\,{$\pm0.7$} & $72.2$\,{$\pm0.7$} \\
\bottomrule\end{tabular}
\caption{{Retention-set size sweep (\cgr{} on LoRA r8).} Accuracy
and CIFAR-100 zero-shot retention are flat over a
$16\times$ range of $|\mathcal{S}_{\mathrm{ret}}|$ on both FGVC-Aircraft and
CUB-200. The default $k{=}128$ sits inside the plateau.}
\label{tab:ksweep}
\end{table*}

\paragraph{Task-set size.}
The confine channel has the matching hyperparameter, the task-set size $|\mathcal{S}{\mathrm{tsk}}|$. It enters both terms, setting the dimension of the subspace the update may occupy and defining the set excluded from $\mathcal{S}{\mathrm{ret}}$ in Eq.\eqref{eq:ret}. Sweeping $|\mathcal{S}{\mathrm{tsk}}|\in{32,64,128,256,512}$ under the Table~\ref{tab:concept_id} protocol (LoRA r8, $\lambda{\mathrm{anc}}{=}60$, $\lambda_{\mathrm{conf}}{=}0.1$, $3$ seeds) leaves accuracy flat over the $16\times$ range on both datasets (Table\ref{tab:stasksweep}). A rank-$8$ update occupies few directions, so even the smallest subspace accommodates it. Retention instead declines monotonically as $|\mathcal{S}{\mathrm{tsk}}|$ grows, by $8.1$ points on FGVC and $5.0$ on CUB-200, through two pathways that move together. A larger task set widens the region the update may occupy, so the perpendicular fraction falls toward zero and the confine term stops binding. It also enlarges the set excluded from $\mathcal{S}{\mathrm{ret}}$, pushing the anchor's candidate pool toward lower-presence concepts. Boundedness therefore matters on the task side exactly as it does on the retention side (Table~\ref{tab:ksweep}). The default $|\mathcal{S}_{\mathrm{tsk}}|{=}128$ is the symmetric untuned choice and sits in the stable middle of the range on both datasets.

\begin{table*}[t]\centering\small\tabcolsep=5pt

\begin{tabular}{ll ccccc}\toprule
 & $|\mathcal{S}_{\mathrm{tsk}}|$ & 32 & 64 & 128 & 256 & 512 \\\midrule
\multirow{3}{*}{FGVC} & Acc & $60.3$\,{$\pm0.6$} & $59.9$\,{$\pm1.3$} & $60.1$\,{$\pm0.4$} & $60.0$\,{$\pm0.7$} & $59.9$\,{$\pm0.5$} \\
 & C-100 ret. & $75.3$\,{$\pm1.3$} & $73.2$\,{$\pm1.0$} & $72.0$\,{$\pm0.8$} & $70.8$\,{$\pm0.4$} & $67.2$\,{$\pm1.1$} \\
\midrule
\multirow{3}{*}{CUB-200} & Acc & $77.0$\,{$\pm0.1$} & $77.0$\,{$\pm0.3$} & $76.4$\,{$\pm0.3$} & $76.1$\,{$\pm0.3$} & $76.0$\,{$\pm0.2$} \\
 & C-100 ret. & $75.7$\,{$\pm1.6$} & $75.0$\,{$\pm0.6$} & $73.2$\,{$\pm0.7$} & $72.0$\,{$\pm1.1$} & $70.7$\,{$\pm0.8$} \\
\bottomrule\end{tabular}
\caption{{Task-set size sweep (\cgr{} on LoRA r8).} Accuracy is
flat over a $16\times$ range of $|\mathcal{S}_{\mathrm{tsk}}|$ on both datasets. Bounding
the task set improves retention at no accuracy cost. The default
$|\mathcal{S}_{\mathrm{tsk}}|{=}128$ matches $|\mathcal{S}_{\mathrm{ret}}|$ and is not
tuned.}
\label{tab:stasksweep}
\end{table*}

\section{SAE faithfulness across VTAB}\label{app:ev}

Table~\ref{tab:app_ev} reports the in-pipeline layer-8 reconstruction EV of the
Prisma SAE on 1000 samples per task used in this paper (frozen backbone, no
training). EV spans $0.87$ (Caltech-101) down to $0.23$ (dSprites).
Natural-image tasks are high-EV (the natural-image SAE is in-distribution, so
its concepts are faithful) while structured/synthetic tasks are low-EV
(out-of-distribution), where the concept channel is expected to weaken. All the mechanism studies above
sit in the high-EV natural regime; the low-EV tasks map the boundary of the
SAE-concept approach. (Notably the $+\cgr$ retention \emph{gain}
tracks forgetting pressure rather than EV directly,
and EV is anti-correlated with that pressure across the natural tasks, so we do
not read EV as a direct cause of the gain.)

\section{Notation}\label{app:notation}

In Table~\ref{tab:notation} we list all notations used in this work.
\begin{table}[t]\centering\small
\begin{tabular}{@{}l p{0.70\columnwidth}@{}}\toprule
Symbol & Meaning \\\midrule
\multicolumn{2}{@{}l}{\emph{Model and adaptation}}\\
$W_0$ & Frozen pretrained weight at the adapted site, $W_0\in\mathbb{R}^{d\times d_{in}}$ \\
$\Delta W$ & Additive weight update; $\Delta W{=}BA$ for LoRA \\
$B,A$ & LoRA factors, $B\in\mathbb{R}^{d\times r}$, $A\in\mathbb{R}^{r\times d_{in}}$ \\
$r$ & LoRA rank \\
$\ell$ & Fixed read-out layer, never tuned \\
$x$ & Input to the adapted weight, $x\in\mathbb{R}^{d_{in}}$ \\
$h_0(x),h_{\mathrm{ft}}(x)$ & Residual-stream vector at the read-out point before and after fine-tuning, $h_{\mathrm{ft}}{=}h_0{+}\Delta W x$ \\
$\Sigma$ & Input second moment $\mathbb{E}[xx^{\top}]$, estimated once on the frozen model \\
$\mathcal{D}_{\mathrm{task}},\mathcal{D}$ & Task input distribution; any evaluation distribution \\
$T$ & Token positions per input (patches for vision) \\\midrule
\multicolumn{2}{@{}l}{\emph{Frozen sparse autoencoder}}\\
$m$ & Dictionary size, $m\gg d$ \\
$W_{enc},b_{enc}$ & Encoder, $W_{enc}\in\mathbb{R}^{m\times d}$; concept $j$ reads through row $W_{enc}[j,:]$ \\
$W_{dec}$ & Decoder, $W_{dec}\in\mathbb{R}^{d\times m}$; concept $j$ writes along column $W_{dec}[:,j]$ \\
$g$ & The SAE's native gate (ReLU, TopK, JumpReLU) \\
$z(h),z_j$ & SAE score vector and its $j$-th entry \\
$f_j(h)$ & \cgr{}'s concept read-out, $\mathrm{ReLU}(W_{enc}[j,:]h+b_{enc}[j])$ \\
$f_{\mathcal{S}}(h)$ & Read-out stacked over a concept set $\mathcal{S}$ \\\midrule
\multicolumn{2}{@{}l}{\emph{\cgr{}}}\\
$\mathcal{S}_{\mathrm{tsk}},\mathcal{S}_{\mathrm{ret}}$ & Task and retention concept sets, disjoint by construction \\
$a_j$ & Causal attribution score selecting $\mathcal{S}_{\mathrm{tsk}}$ \\
$\phi_j$ & Firing presence selecting $\mathcal{S}_{\mathrm{ret}}$ \\
$U_{\mathrm{tsk}}$ & Orthonormal basis of the task subspace, $U_{\mathrm{tsk}}\in\mathbb{R}^{d\times|\mathcal{S}_{\mathrm{tsk}}|}$ \\
$P_{\mathrm{tsk}},P^{\perp}_{\mathrm{tsk}}$ & Projector onto the task subspace and its complement \\
$\Lconf,\Lanchor$ & \textsc{confine} and \textsc{anchor} penalties \\
$\lambda_{\mathrm{conf}},\lambda_{\mathrm{anc}}$ & Their coefficients in Eq.~\eqref{eq:total} \\
{$\mathcal{L}_{\mathrm{task}},\mathcal{L}$} & {Fine-tuning loss and the total objective of Eq.~\eqref{eq:total}} \\
{$k$} & {Retention-set size, $k=|\mathcal{S}_{\mathrm{ret}}|$ in Eq.~\eqref{eq:ret}} \\\midrule
\multicolumn{2}{@{}l}{\emph{Theory}}\\
$\Delta f_{\mathrm{ret}}(x)$ & Retention read-out drift, $f_{\mathcal{S}_{\mathrm{ret}}}(h_{\mathrm{ft}}){-}f_{\mathcal{S}_{\mathrm{ret}}}(h_0)$ \\
$J_{\mathrm{ret}}$ & Jacobian of the retention read-out, $\partial f_{\mathcal{S}_{\mathrm{ret}}}/\partial h$ \\
$R,\Delta R$ & Retention functional and its change (a frozen probe's margin or loss, not the reported accuracy) \\
$L_R$ & Lipschitz constant of $R$ \\
$W_{ret},b_{ret}$ & Retention read-out weights, $W_{enc}[\mathcal{S}_{\mathrm{ret}},:]$ and $b_{enc}[\mathcal{S}_{\mathrm{ret}}]$ \\
{$\mathcal{L}_{\mathrm{low}},\mathcal{L}_{\mathrm{up}}$} & {Weight-space drift bounds, Eq.~\eqref{eq:drift-bounds}} \\
{$\mathcal{J},W_{\mathcal{J}}$} & {Always-active subset $\mathcal{J}\subseteq\mathcal{S}_{\mathrm{ret}}$ in Prop.~\ref{prop:drift-bounds} and its read-out rows $W_{\mathrm{ret}}[\mathcal{J},:]$} \\
{$\Sigma_A$} & {Second moment restricted to the input region $A$, $\mathbb{E}[\mathbf{1}_A(x)\,xx^{\top}]$} \\
\bottomrule\end{tabular}
\caption{Notation used throughout the paper.}\label{tab:notation}
\end{table}

\end{document}